\documentclass[a4paper, 12pt]{article}

\usepackage[T1]{fontenc}
\usepackage[utf8]{inputenc}
\usepackage{amsmath}
\usepackage{amssymb}
\usepackage{amsthm}
\usepackage{tikz}
\usepackage{graphicx}
\usepackage{pgfplots}
\usepackage{url}
\usepackage{xcolor}
\usepackage{dsfont}
\usepackage{xfrac}
\usepackage[pagebackref=true]{hyperref}
\usepackage{mathtools}
\usepackage{url}
\usepackage{eufrak}
\usepackage{bbold}
\usepackage{parskip}
\usepackage{aliascnt}
\usepackage{cleveref}
\usepackage{subcaption}
\usepackage{ulem}
\usepackage{tikz}

\usepackage[a4paper, total={7in, 7in}]{geometry}

\usetikzlibrary{trees}
\usetikzlibrary{matrix}
\usetikzlibrary{positioning,fit,backgrounds}

\pgfdeclarelayer{nodelayer}
\pgfdeclarelayer{edgelayer}
\pgfsetlayers{edgelayer,nodelayer,main}
\tikzset{none/.style={thick}}

\def\P{\mathbb{P}}
\def\one{\mathds{1}}

\newtheorem{theorem}{Theorem}
\newtheorem{proposition}[theorem]{Proposition}%
\newtheorem{corollary}[theorem]{Corollary}

\newtheorem{example}{Example}%
\newtheorem{remark}{Remark}%
\newtheorem{lemma}{Lemma}%

\numberwithin{equation}{section}

\title{Private Quantiles Estimation in the Presence of Atoms}

\author{
  Clément Lalanne\\
  LIP, Univ Lyon, EnsL, UCBL, CNRS, Inria, LYON Cedex 07 F-69342,France\\
  \texttt{clement.lalanne@ens-lyon.fr} 
  \and
  Clément Gastaud\\
  Sarus Technologies SAS, 128 rue la Boétie, 75008 Paris, France
  \and
  Nicolas Grislain\\
  Sarus Technologies SAS, 128 rue la Boétie, 75008 Paris, France
  \and
  Aurélien Garivier\\
  UMPA UMR 5669, Univ. Lyon, ENS de Lyon, 46 allée d'Italie, \\Lyon cedex 07 F-69364, France
  \and
  Rémi Gribonval\\
  LIP, Univ Lyon, EnsL, UCBL, CNRS, Inria, LYON Cedex 07 F-69342,France
}

\begin{document}

\maketitle





\abstract{ We consider the differentially private estimation of multiple quantiles (MQ) of a distribution from a dataset, a key building block in modern data analysis. 
	We apply the recent non-smoothed Inverse Sensitivity (IS) mechanism to this specific problem. We establish that the resulting method is closely related to the recently published ad hoc algorithm \emph{JointExp}. In particular, they share the same computational complexity and a similar efficiency.
	We prove the statistical consistency of these two algorithms for continuous distributions.
    Furthermore, we demonstrate both theoretically and empirically that this method suffers from an important lack of performance in the case of peaked distributions, which can degrade up to a potentially catastrophic impact in the presence of atoms. 
    Its smoothed version (i.e. by applying a max kernel to its output density) would solve this problem, but remains an open challenge to implement. 
    As a proxy, we propose a simple and numerically efficient method called Heuristically Smoothed JointExp (\emph{HSJointExp}), which is endowed with performance guarantees for a broad class of distributions and achieves results that are orders of magnitude better on problematic datasets. }


\maketitle

\section{Introduction}

As more and more data is collected on individuals and data science techniques become more powerful, threats to privacy have multiplied and serious concerns have emerged~\cite{narayanan2006break,backstrom2007wherefore,fredrikson2015model,dinur2003revealing,homer2008resolving,loukides2010disclosure,narayanan2008robust,sweeney2000simple,wagner2018technical,sweeney2002k}. Against this background, \emph{differential privacy} (DP)~\cite{dwork2014algorithmic} has become the gold standard in privacy protection. By introducing 
randomness calibrated to the {\em sensitivity} of a query~\cite{dwork2006calibrating}, it enables the inference of global statistics on a dataset while bounding each sample's influence and ensuring that the presence or absence of an individual in the dataset cannot be deduced from the result. In the last decade, research results have brought nice building blocks and composition theorems \cite{dwork2006calibrating,kairouz2015composition,dong2019gaussian,dong2020optimal,abadi2016deep}. They paved the way for many applications in data analysis, from basic statistics to advanced artificial intelligence algorithms. Notably, differential privacy is now used in production by the US Census Bureau~\cite{abowd2018us}, Google~\cite{erlingsson2014rappor}, Apple~\cite{thakurta2017learning} and Microsoft~\cite{ding2017collecting} among others.

In this paper, we focus on the problem of estimating one or many \emph{quantiles} with privacy guarantees. Beyond the interest that quantiles have in themselves, they are also important primitives in many advanced applications in machine learning, from synthetic data generation to decision tree training. Indeed, since quantiles are a reasonable choice of bins for quantizing a cumulative distribution function, they are commonly used in algorithms based on decision trees, such as Random Forests and Boosted Trees \cite{chen2016xgboost}, where variables are binned using quantiles before they are considered for a split. Besides, in many recent synthetic data models, continuous features are binned to reduce the output space's dimension (see~\cite{zhang2017privbayes, mckenna2021winning}).

Given a real random variable $X$ of probability distribution $\mathbb{P}_X$, the cumulative distribution function (CDF) of $X$ (or $\mathbb{P}_X$), noted $F_X$ (or $F_{\mathbb{P}_X}$) is classically defined as $F_X(t) = \mathbb{P}_X(X \leq t), \forall t \in \mathbb{R}$ \;.
Its pseudo-inverse, the quantile function, $F_X^{-1}$ (or $F_{\mathbb{P}_X}^{-1}$) is defined as 
\begin{equation*}
    F_X^{-1}(p) = \inf \left\{ t \in \mathbb{R} | F_X(t) \geq p \right\}, \quad \forall p \in [0, 1] \;,
\end{equation*}
with the convention $\inf \emptyset = + \infty$. The quantity $F_X^{-1}(p)$ is the \emph{quantile} of order $p$ of the distribution of $X$. Furthermore, when given a dataset (a collection of real numbers), the \emph{empirical quantile} of order $p$ of this dataset is the quantile of order $p$ of its empirical distribution. 

When privacy is not an issue, it is well known that the empirical quantiles of a dataset of i.i.d. random variables 
are good estimators of the quantiles of the underlying distribution \cite{van2000asymptotic}. There has been a lot of recent research on how to privately estimate multiple empirical quantiles from a dataset with a privacy overhead as small as possible. This article builds on this framework by considering private estimates of the empirical quantiles as estimators of the quantiles of the true distribution.

Among several ways to privately estimate empirical quantiles, 
a naive one is to add Laplace noise to non-private quantile estimates \cite{dwork2006calibrating}. It is straightforward and easy to compute, but the amount of added noise is based on a pessimistic scenario that 
cannot materialize simultaneously for all quantiles. 
To reduce the variance of the estimates, a variant that uses so-called \emph{smoothed sensitivity} instead of the worst case scenario was introduced \cite{nissim2007smooth}. It has the drawback, however, of using approximate differential privacy \cite{dwork2006our} instead of pure differential privacy, which allows some catastrophic failures with small probability. 
The current state of the art for single quantile estimation uses a fine-tuned version of the standard exponential mechanism for differential privacy \cite{mcsherry2007mechanism} that is called \emph{ExponentialQuantile} \cite{smith2011privacy}. It is cheap to compute and the recent theory of \emph{Inverse Sensitivity} \cite{asi2020near} shows that a conceptually simple smoothing operation allows levels of privacy that are quadratically better than those guaranteed for the smoothed sensitivity approaches while sticking to pure differential privacy, achieving near-instance optimality (i.e. with a risk comparable to the local minimax risk when the set of hypotheses is the set of neighboring datasets). As a result, it has been adopted by the main DP-ready software libraries  \cite{smartnoise,ibmdiffpriv} used in production.

All the algorithms previously mentioned are designed to estimate a single quantile from a dataset. Composition theorems make them usable for multiquantile estimation by evaluating each quantile independently, but with an increased privacy noise. Recent work by \cite{gillenwater2021differentially} has presented a new algorithm, \emph{JointExp}, that was the first to exploit the non-decreasing constraint of quantiles at the core of its sampling procedure and became the empirical state of the art for a small period. Even more recent work \cite{kaplan2022differentially} cleverly exploits the structure of quantiles by computing them recursively / hierarchically on disjoint subsets of the dataset, which reduces the privacy overhead in composition. It is the current empirical state of the art for estimating many quantiles.

\subsection{Contributions and organization of the paper}

In this paper, we study the use of the Inverse Sensitivity mechanism \cite{asi2020near,asi2021nips} in order to estimate multiple empirical quantiles from a dataset, and their statistical properties as estimators of the quantiles of the underlying distribution. 

In \Cref{section1} we recall the needed background on differential privacy, Inverse Sensitivity, and JointExp.
\Cref{sec:link} is devoted to characterizing the Inverse Sensitivity mechanism for private multiquantile estimation. In particular, our first contribution is to obtain the precise expression of the utility function of the Inverse Sensitivity mechanism for an arbitrary number of empirical quantiles, which was previously known for a single quantile only. In particular, this expression is quite similar to the utility function of JointExp \cite{gillenwater2021differentially}, which was proposed as an ad-hoc algorithm based on a heuristic. We draw the explicit link between the two algorithms. Due to their similarities, the two algorithms are used interchangeably in most of the article.

It was noted by the authors of JointExp \cite{gillenwater2021differentially} that their algorithm struggles when the gaps between the data points are too small, but without more details. Our next contribution is to statistically quantify this empirical phenomenon by proving that JointExp/IS is in fact inconsistent on \emph{atomic distributions}. This result is presented in \Cref{sec:failure}.

\Cref{hssmoothing} serves two purposes : we prove the consistency of JointExp/IS on \emph{continuous distributions} (i.e. with a continuous density w.r.t. Lebesgue measure). This is the first consistency result of JointExp and is thus a big step towards the understanding of this algorithm. 
Furthermore, we propose a new heuristic, tractable smoothing technique based on \emph{jittering} for JointExp/IS that vastly improves their utility on peaked distributions without noticeable repercussions on non-degenerate distributions. In particular, this estimator is consistent on atomic distributions. Our technique differs from the general smoothing trick for \emph{inverse sensitivity} based mechanisms introduced by \cite{asi2020near,asi2021nips}, a conceptual trick consisting in taking a maximum convolution of the density over the output space, in the sense that our technique is a concrete mechanism to smooth 
the data distribution. As such, our technique results 
in a much better computational complexity making it a viable solution in high dimension. While similar techniques have already been used in order to fix ill-posed problems (see \cite{owen2001empirical,machado2005quantiles,chen2010quantile}) when dealing with peaked distributions, the motivations behind the addition of jitter here are fairly different : it fixes an over-penalization of the exponential mechanism when dealing with concentrated data.

Finally, \Cref{section5} gathers the numerical experiments demonstrating the behavior of the proposed algorithms.

\section{Background}
\label{section1}
Considering datasets of the form $\mathbf{X} = (X_1,\dots,X_n)\in\mathfrak{X}^n$, where $\mathfrak{X}$ denotes our feature space and $n\geq 1$ is the sample size, the \emph{Hamming distance} $d(\mathbf{X}, \mathbf{Y})$ between two datasets $\mathbf{X}, \mathbf{Y}\in\mathfrak{X}^n$ is defined as the minimal number of changes (i.e.,
substitutions of entries) required to transform $\mathbf{X}$ into a permutation of $\mathbf{Y}$. Hence,  $d(\mathbf{X}, \mathbf{Y}) = 0$ i.f.f. there exists a permutation $\sigma$ on $\{1,\dots,n\}$ such that $ \forall i \in \left\{1, \dots, n\right\}, X_i = Y_{\sigma(i)}$. We say that $\mathbf{X}$ and $\mathbf{Y}$ are \emph{neighbors} (noted $\mathbf{X} \sim \mathbf{Y}$) if $d(\mathbf{X}, \mathbf{Y})=1$, that is if there exist two permutations $\sigma_1, \sigma_2$ such that $\forall i \in \left\{1, \dots, 
n-1\right\}, X_{\sigma_1(i)} = Y_{\sigma_2(i)}$. Note that  $d(\mathbf{X}, \mathbf{Y})$ is the minimum length of a path on consecutive neighbors linking $\mathbf{X}$ to $\mathbf{Y}$.

\subsection{Differential Privacy}

Given a privacy budget $\epsilon>0$, a randomized algorithm $\mathcal{A}:\mathfrak{X}^n \to O$ is called \emph{$\epsilon$-differentially private} ($\epsilon$-DP, see~\cite{dwork2006our}) if  for all pairs of datasets $\mathbf{X}, \mathbf{Y} \in \mathfrak{X}^n$ and all measurable sets $S \subseteq O$,
\begin{equation*}
    \mathbf{X} \sim \mathbf{Y} \Rightarrow  \mathbb{P} \big( \mathcal{A}(\mathbf{X}) \in S \big) \leq e^\epsilon \times \mathbb{P}\big( \mathcal{A}(\mathbf{Y}\big) \in S)\;.
\end{equation*}
Differential privacy offers strong privacy protections by bounding the efficiency of any test trying to distinguish two neighboring databases.
A classical way to design $\epsilon$-DP algorithms is the Exponential Mechanism~\cite{mcsherry2007mechanism}. For a utility function  $u: \mathfrak{X}^n \times O \rightarrow \mathbb{R}$ that measures the relevance $u(\mathbf{X}, o)$ of the output $o$ for the dataset $\mathbf{X}$, the Exponential Mechanism $\mathcal{E}_{u}^{(\alpha)}$ defined by $u$ and  with parameter $\alpha>0$ outputs a random variable on $O$ with a density proportional to $e^{u(\mathbf{X}, o)/ \alpha}$ with respect to some reference measure on $O$.  For example, when $O$ is discrete, 
\[        
    \mathbb{P}\left(\mathcal{E}_{u}^{(\alpha)}(\mathbf{X}) = o \right) 
    = \frac{e^{u(\mathbf{X}, o)/ \alpha}}{\sum_{o'\in O}e^{u(\mathbf{X}, o')/ \alpha}}   \quad \forall o\in O \;.
\]
Defining the \emph{sensitivity} of the utility function 
    \begin{equation*} 
        \Delta u := \sup_{o \in O, \mathbf{X}, \mathbf{Y} \in \mathfrak{X}^n: \mathbf{X} \sim \mathbf{Y}} |u(\mathbf{X}, o) - u(\mathbf{Y}, o)| \;,
    \end{equation*}
 a classical result is that $\mathcal{E}_{u}^{(\alpha)}$ is $\epsilon$-DP as soon as $\alpha \geq \frac{2 \Delta u}{\epsilon}$ 
 \cite{mcsherry2007mechanism}. All the algorithms presented in this paper build on this mechanism.

\subsection{The inverse sensitivity mechanism}

When considering some deterministic function $\mathcal{Q}: \mathfrak{X}^n \rightarrow O$ as the target of privatization, a specific choice of utility function in the exponential mechanism, long known as folklore, was proved to have remarkable optimality properties under certain assumptions \cite{asi2020near}. The \emph{inverse sensitivity function} 
\begin{equation*}
    u_{\text{IS}}(\mathbf{X}, o) := -\inf \Big\{ d(\mathbf{X}, \mathbf{Y}) \mbox{ s.t. }  \mathbf{Y} \in \mathcal{Q}^{-1}(o) \Big\} \;
\end{equation*}
is easily seen to have
 sensitivity $\Delta u_{\text{IS}} = 1$.
The resulting $\epsilon$-DP mechanism $\mathcal{E}_{u_{\text{IS}}}^{(2/\epsilon)}(\mathbf{X})$ has then a behavior that is quite intuitive: the likelihood of its output $o$ decreases when $\mathcal{Q}^{-1}(o)$ becomes further apart from $\mathbf{X}$ in the Hamming distance, which means that the probability of an output decreases the more points have to be modified in the dataset for it to be a viable deterministic output.

\subsection{JointExp}

Specifying the notations for the multiquantile estimation, given $a <b \in \mathbb{R}$, let $\mathfrak{X} = [a, b]$ be the feature space and let $O = [a, b]^m\!\!\!\!\!\nearrow $ be the set of vectors of $m$ increasing points in $[a, b]$ representing $m$ quantiles. 
The hypothesis that the feature space is bounded is necessary to the analysis and is reasonable for many applications. For applications where this is unrealistic, solutions have been proposed at the expense of having an algorithm that has a small chance of halting \cite{dwork2009differential,brunel2020propose}. Given a probability vector $\mathbf{p} = (p_1, \dots, p_m) \in (0, 1)^m\!\!\!\!\!\nearrow$, the goal is to estimate the empirical quantile function associated to $\mathbf{p}$ and defined as 
\begin{equation*}
    \begin{array}{ccccc}
        \mathcal{Q} & : & \mathfrak{X}^n & \to & O \\
         & & \mathbf{X} & \mapsto & (X_{(\lceil n p_1\rceil)}, \dots, X_{(\lceil n p_m\rceil)}) \\
        \end{array}
\end{equation*}
where $X_{(i)}$ denotes the $i$-th order statistics of $X_{1},\ldots,X_{n}$.
As a safety check, we assume that $\forall j \in \{1, \dots, m-1\}, n(p_{j+1}-p_j) \geq 1$ to ensure that no data point will be chosen twice as a quantile representant. Note that given $\mathbf{p}$, this condition can always be satisfied provided that we have enough data, i.e., that $n$ is large enough. For any $\mathbf{X} \in \mathfrak{X}^n$, $\mathbf{q} \in O$ and $\mathbf{p} \in (0,1)^m$, we use the convention $X_{i \leq 0} = q_{i \leq 0} = a$, $X_{i \geq n+1} = q_{i \geq m+1} = b$, $p_{i \leq 0} = 0$ and $p_{i \geq m+1} = 1$. Finally, for the brevity of notation, vectors are interpreted when needed as the set containing their components. 


For reasons that will become clear later, we take the time to redefine the JointExp \cite{gillenwater2021differentially} (also called ExponentialQuantile when $m=1$) mechanism. It corresponds to
a specific instantiation of the exponential mechanism $\mathcal{E}_{u_{\text{JE}}}^{(2/\epsilon)}(\mathbf{X})$ with
\begin{align*}
    -u_{\text{JE}}(\mathbf{X}, \mathbf{q}) &:= \frac{1}{2} \sum_{i=1}^{m+1} 
\left| \delta^{\text{JE}}(i,\mathbf{X},\mathbf{q})\right| \;,\\
\intertext{(which is of sensitivity $1$) where}
    \delta^{\text{JE}}(i,\mathbf{X},\mathbf{q}) & := n(p_i-p_{i-1})- \#\left(\mathbf{X} \cap (q_{i-1}, q_i] \right) \;.
\end{align*}
This mechanism works by penalizing the result whenever the number of data points in each quantile interval ($\#\left(\mathbf{X} \cap (q_{i-1}, q_i] \right)$) deviates from what should be expected ($n(p_i-p_{i-1})$).

\section{JointExp meets Inverse Sensitivity}
\label{sec:link}

At first glance, there is no connection between the theory of the Inverse Sensitivity and JointExp. The first one is born from the need to build a general mechanism that is endowed with optimality properties \cite{asi2020near,asi2021nips} for a broad class of problems, while the second comes from the idea that good empirical quantiles should separate the data points proportionally. In the case of the estimation of a single quantile (i.e. $m=1$), it was observed \cite{asi2020near} that the two algorithms are similar. Our contribution in this section is to prove that, up to minor differences, this remains true with an arbitrary number of quantiles. For this, we provide the precise expression of the inverse sensitivity function for the multiquantile problem.

Deriving the expression of the inverse sensitivity for a dataset $\mathbf{X}$ and an output candidate $\mathbf{q}$ boils down to answering the question: What is the minimal number of points from $\mathbf{X}$ that need to be changed in order to obtain a vector that has $\mathbf{q}$ as its empirical quantiles? \Cref{mainth} solves this question for Lebesgue-almost-any $\mathbf{q}$.
\begin{theorem}
    \label[theorem]{mainth}
    For any $\mathbf{X}\in \mathfrak{X}^n$ and $\mathbf{q} \in \left([a, b] \setminus \mathbf{X} \right)^m \!\!\!\!\!\nearrow$ without collision, 
    \begin{equation*}
        \begin{aligned}
            - u_{\text{IS}}(\mathbf{X}, \mathbf{q}) &= 
            \frac{1}{2} \sum_{i=2}^{m+1} 
            \left| \delta(i,\mathbf{X},\mathbf{q}) \right| 
            + \sum_{i=2}^{m} \mathds{1}_{\mathbb{R}_+}\left(\delta(i,\mathbf{X},\mathbf{q})\right) 
            \\
            &+ 
            \frac{1}{2} \left| \delta_{\texttt{closed}}(1,\mathbf{X},\mathbf{q}) \right| 
            + \mathds{1}_{\mathbb{R}_+}\left(\delta_{\texttt{closed}}(1,\mathbf{X},\mathbf{q})\right) 
        \end{aligned}
    \end{equation*}
    with 
    \begin{align*}
        \delta(i,\mathbf{X},\mathbf{q}) &=
        \#\left(\mathbf{X} \cap (q_{i-1}, q_i] \right) - \left(\lceil np_i \rceil - \lceil np_{i-1} \rceil\right)\\
        \delta_{\texttt{closed}}(i,\mathbf{X},\mathbf{q}) 
        &= \#\left(\mathbf{X} \cap [q_{i-1}, q_i] \right) - \left(\lceil np_i \rceil - \lceil np_{i-1} \rceil\right).
    \end{align*}
\end{theorem}
We postpone the proof to \Cref{mainthproof} for brevity.
The case when $\mathbf{q}$ has collisions or shares some common points with the dataset is more difficult. Luckily, those cases can be neglected when considering the sampling mechanism.
Indeed, $\mathcal{E}_{u_{\text{IS}}}^{(2/\epsilon)}(\mathbf{X})$ has a density that is absolutely continuous w.r.t. Lebesgue measure, the expression of the resulting mechanism can be further simplified (see \Cref{simplifiedMech}) by modifying the density on outcomes of null Lebesgue measure.
\begin{corollary}
\label[corollary]{simplifiedMech}
    For any $\mathbf{X} \in \mathfrak{X}^n$, 
    $\mathcal{E}_{u_{\text{IS}}}^{(2/\epsilon)}(\mathbf{X})$ has the same output distribution as 
    $\mathcal{E}_{\tilde{u}_{\text{IS}}}^{(2/\epsilon)}(\mathbf{X})$ where
    $\forall \mathbf{X} \in \mathfrak{X}^n$, $\forall \mathbf{q} \in O$, 
    \begin{equation*}
        \begin{aligned}
             - \tilde{u}_{\text{IS}}(\mathbf{X}, \mathbf{q}) &=
            \frac{1}{2} \sum_{i=1}^{m+1} 
            \left| \delta(i,\mathbf{X},\mathbf{q})\right| 
            + \sum_{i=1}^{m} \mathds{1}_{\mathbb{R}_+}(\delta(i,\mathbf{X},\mathbf{q})) \;.
        \end{aligned}
    \end{equation*}
\end{corollary} 

\begin{remark}
\label[remark]{fqzseffdsqsf}
We discuss the sampling from $\mathcal{E}_{\tilde{u}_{\text{IS}}}^{(2/\epsilon)}(\mathbf{X})$ in \Cref{sampleis} but our conclusion is that it can be done by making some minor adjustments to the JointExp sampling algorithm and that, in particular, the two algorithms share the same complexity of $O(n m \log n + nm^2)$.
\end{remark}

\begin{remark}
\label[remark]{linkJEIS}
One can check that
$|\tilde{u}_{\text{IS}}(\mathbf{X}, \mathbf{q}) - u_{\text{JE}}(\mathbf{X}, \mathbf{q})| \leq 2(m+1)$ and thus the distributions differ significantly only on outcomes of high utility (when the number of misclassified points is of the order $O(m)$). The bad outcomes are almost equally penalized and for this reason, we can expect the two algorithms to perform almost identically when $n$ is large enough.   This is indeed confirmed by numerical examples, as illustrated in \Cref{section5}. As a consequence, we will mainly focus on JointExp for the rest of this article, all the results being applicable to IS as well (with some minor tweaks).
\end{remark}

\section{JointExp fails on atomic distributions}
\label{sec:failure}

To the best of our knowledge, no theoretical utility guarantee for JointExp has been derived yet, and the performance of this algorithm has only been demonstrated experimentally. Even if it outperforms by multiple orders of magnitude previous techniques on many real life datasets~\cite{gillenwater2021differentially}, we prove in this section that it can also completely fail on some distributions (see \Cref{JointExpFailureTheorem}).
As illustrated in \Cref{section5}, JointExp is indeed observed to be suboptimal on several real world datasets associated to peaked distributions, such as the US Census Bureau "Dividends" and "Earnings" data. 

In order to understand the origin of this weakness of JointExp, we analyse the density of the distribution of its output. This density is constant on the 
\emph{``blocks''}
$\big([X_{i_1}, X_{i_1 + 1}) \times \ldots
\times \dots \times [X_{i_{m}}, X_{i_m +1}) \big) \cap [a, b]^m \!\!\!\!\!\nearrow$ for each $\mathbf{i}=(i_1, \dots, i_m) \in O'$ where
$O'= \left\{\mathbf{i} \in \left\{0, \dots, n\right\}^m, 
    i_1\leq \dots \leq i_m 
    \right\}$ \;.

The probability of the output of JointExp being in a given block
is proportional to the volume of this block.
What can happen in practice is that even though a block
is interesting in terms of utility level, its volume can in fact be close to zero if the data points are close. The volume can even be zero in case of equality, hence this block is never selected by the exponential mechanism.
This phenomenon occurs particularly often for data drawn from distributions with isolated atoms: asymptotically, the dataset will almost surely contain collisions among the data points as $n$ grows and JointExp will fail on the corresponding quantiles.

To formally capture this phenomenon, from now on, $\mathbf{X}$ is supposed to be a collection of $n$ i.i.d. samples of a random variable $X$ with distribution $\mathbb{P}_{X}$ and with cumulative distribution function (CDF) $F_{X}$.
\begin{proposition}
    \label[proposition]{JointExpFailureTheorem}
    Suppose that there exist 
    $q \in (a,b)$ and $\eta>0$ such that $I := (q-\eta,q+\eta) \subset [a,b]$ satisfies
    $\mathbb{P}_{X}(\{q\})>0$ and $\mathbb{P}_{X}(I \setminus \{q\})=0$. 
    Then there exist some probability vectors $\mathbf{p}$ such that 
    \begin{equation}
    \label{eqpb}
        \mathbb{E}_{\mathbf{X}, \mathcal{E}_{u_{\text{JE}}}^{(2/\epsilon)}} \left( \|\mathcal{E}_{u_{\text{JE}}}^{(2/\epsilon)}(\mathbf{X}) - F_{X}^{-1}(\mathbf{p})\|_\infty \right) = \Omega_n (1) \;,
    \end{equation}
    where we use the vector notation $F_{X}^{-1}(\mathbf{p}) = (F_{X}^{-1}(p_1), \dots, F_{X}^{-1}(p_m))$. Furthermore, the Lebesgue measure of the set of problematic probability vectors is lower bounded by $\mathbb{P}_{X}(\{q\})^m / (m!)$.
\end{proposition}
$\Omega_n(1)$ refers to a quantity that is lower bounded by a positive constant when $n$ grows.
We postpone the proof to \Cref{JointExpFailureTheoremproof} for brevity.

This result shows that for certain data distributions with isolated atoms, JointExp is not consistent, even asymptotically, on many instances of the estimation problem (i.e. not on unrealistic corner cases). This behavior is 
all the more counterintuitive as one would think that on datasets with a lot of collisions, very little noise would be needed to ensure privacy since the points are already indistinguishable.

\begin{example}\label[example]{exa:constant}
Consider the private estimation of the median (i.e. $m=1$ quantile, and $\mathbf{p} = (1/2)$) on $[a,b] = [-1,1]$. Since $m=1$, JointExp coincides with ExponentialQuantile, and 
when all data points are equal to $0$ (i.e. $\mathbb{P}_X = \delta_{0}$) its output is uniformly distributed in $[-1,1]$ \emph{whatever the sample size} $n$ as long as it is even. 
\end{example}

When considering estimation on real-world distributions, many real-life datasets show \emph{accumulation points} and can be modeled as continuous distributions with
some Diracs at specific points. A famous example is the revenue statistics of the US Census Bureau:
many participants in surveys are not qualified to have some category of revenue (too young or not investing in some assets) hence the presence of accumulations at the zero value for these categories.
In fact, any continuous variable that is censored, conditional on some other variable or generated by mimetic agents tending
to repeat exactly some values, will show accumulation points where JointExp has great chances to fail.

\section{Heuristic smoothing, with guarantees}
\label{hssmoothing}

The type of failure of JointExp highlighted in \Cref{sec:failure} may seem surprising given a) the strong connection between JointExp and the Inverse Sensitivity established in \Cref{sec:link}; and b) existing performance guarantees for \emph{smoothed} Inverse Sensitivity mechanisms \cite{asi2020near,asi2021nips}.
Indeed, while JointExp is not smoothed, smoothing convolves the output distribution with a max kernel, increasing the volume of the maximum of the distribution to circumvent the difficulties raised by isolated atoms. We discuss such an approach in \Cref{ISSmoothing} and conclude that while it would increase the utility of the resulting mechanism, it would also make it computationally intractable.

As a tractable alternative, we present in this section a heuristic algorithm based on noise addition prior to the application of JointExp, and we show that this mechanism is endowed with privacy and consistency guarantees. Note that the exposed problems with atomic distribution also occur for highly concentrated continuous distributions. Hence simpler and more naive solutions such as multi-indices do not fix them.

Another possible solution would be to discretize the output space. However, the resulting algorithm would have a complexity of $O(f(m, n, \delta) + 1/\delta^m)$ where $\delta$ is the precision of the discretization and $f$ is some function. Since this is exponential in the number of quantiles, it suffers from the curse of dimensionality, and we argue that jittering is a better alternative.

\subsection{Introducing the HSJointExp algorithm}

Since JointExp has a density that is constant on the 
blocks
$\big([X_{i_1}, X_{i_1 + 1}) \times \ldots
\times \dots \times [X_{i_{m}}, X_{i_m +1}) \big) \\ \cap [a, b]^m \!\!\!\!\!\nearrow$ for $\mathbf{i}=(i_1, \dots, i_m) \in O'$,
it fails when the blocks that have a great utility (i.e. the ones leading to interesting quantile candidates) have a volume that is too small. By adding noise to the data points, we ensure a minimal volume for the blocks, and in particular for the interesting regions, while only shifting the empirical quantiles of the dataset by a small amount.

Let $w_1, \dots, w_n$ be i.i.d variables, and let
\begin{equation}\label{eq:NoisyDataset}
\tilde{\mathbf{X}} = (X_1 + w_1, \dots, X_n + w_n)\;.
\end{equation}
The Heuristically Smoothed JointExp (HSJointExp) is defined as the algorithm that returns  $\mathcal{E}_{u_{\text{JE}}}^{(2/\epsilon)}(\tilde{\mathbf{X}})$, the output of the JointExp on the noisy data $\tilde{\mathbf{X}}$.

Let us now discuss the choice of the distribution $\mathbb{P}_w$ of the $(w_i)'s$. 
Discrete noise distributions (for instance Bernoulli noise scaled by some $\alpha > 0$: $\frac{w}{\alpha} \sim \mathcal{B}(\frac{1}{2})$) may seem interesting because they lead to easily tuneable data gaps. However, this often just creates new instances where JointExp fails. Indeed, adding discrete noise to data distributions with accumulation points creates new accumulation points. 

For this reason, we focus in the sequel on continuous noise distributions with a density denoted by $\pi_{w}$. The density $\pi_{\tilde{X}}$ of the noisy data $\tilde{X}$ is hence given by the convolution formula,
\begin{equation}
    \forall t \in \mathbb{R}, \quad \pi_{\tilde{X}}(t) = \int \pi_{w}(t-x) \mathbb{P}_X(dx) \;.
\end{equation}
A typical choice of noise discussed in the sequel is the uniform distribution on the interval $[-\alpha, \alpha]$. 

Before discussing the choice of the scale parameter $\alpha$, we remark that HSJointExp  consists of the addition of i.i.d. noise prior to running JointExp. Its privacy guarantees are thus a direct consequence of the following generic composition lemma. Its proof, which we did not find elsewhere, is in the supplementary material (see \Cref{preprocessinglemmaproof}).
\begin{proposition}
    \label[proposition]{preprocessionglemma}
    Let $\mathbf{w}$ be a random variable on $\mathbb{R}^n$ with probability distribution $\mathbb{P}_{\mathbf{w}}$ that is invariant by permutations of the components of the vector. If 
    $\mathcal{A}$ is $\epsilon$-DP on $\mathfrak{X}^n$, then $\mathbf{X} \mapsto \mathcal{A}(\texttt{proj}_{\mathfrak{X}^n}(\mathbf{X}+\mathbf{w}))$ is also $\epsilon$-DP.
\end{proposition}
The projection step $\texttt{proj}$ onto the data space $\mathfrak{X}^n$ is necessary because JointExp needs to know the range of the data. Note that $\mathfrak{X}^n$ could be replaced by any set of the form $[a-\delta_{\alpha, n},b + \delta_{\alpha, n}]^n$ where $\delta_{\alpha, n}$ is a quantity that depends on $\alpha$ and $n$. So for instance, if the noise follows a uniform distribution on the interval $[-\alpha, \alpha]$, projecting on $[a-\alpha,b + \alpha]^n$ (does nothing) and then running JointExp on $[a-\alpha,b + \alpha]$ ensures that no point will overflow.

\subsection{Consistency of HSJointExp on constant data}
\label{HSJointExpcst}

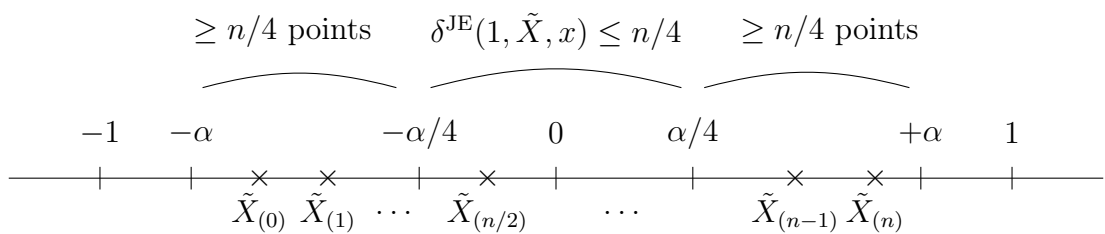
\begin{figure*}[h!]    
    \centering
    \begin{tikzpicture}[scale=0.6]
	\begin{pgfonlayer}{nodelayer}
		\node [style=none] (0) at (0, 1) {$0$};
		\node [style=none] (1) at (8, 1) {$+\alpha$};
		\node [style=none] (2) at (-8, 1) {$-\alpha$};
		\node [style=none] (3) at (-12, 0) {};
		\node [style=none] (4) at (12, 0) {};
		\node [style=none] (5) at (-8, 0.25) {};
		\node [style=none] (6) at (-8, -0.25) {};
		\node [style=none] (7) at (0, 0.25) {};
		\node [style=none] (8) at (0, -0.25) {};
		\node [style=none] (9) at (8, 0.25) {};
		\node [style=none] (10) at (8, -0.25) {};
		\node [style=none] (12) at (-10, 1) {$-1$};
		\node [style=none] (13) at (10, 1) {$1$};
		\node [style=none] (14) at (-10, 0.25) {};
		\node [style=none] (15) at (-10, -0.25) {};
		\node [style=none] (16) at (10, 0.25) {};
		\node [style=none] (17) at (10, -0.25) {};
		\node [style=none] (18) at (-3, 0.25) {};
		\node [style=none] (19) at (-3, -0.25) {};
		\node [style=none] (20) at (3, 0.25) {};
		\node [style=none] (21) at (3, -0.25) {};
		\node [style=none] (22) at (-3, 1) {$-\alpha/4$};
		\node [style=none] (23) at (3, 1) {$\alpha/4$};
		\node [style=none] (24) at (-1.5, -0.75) {$\tilde{X}_{(n/2)}$};
		\node [style=none] (25) at (-6.5, -0.75) {$\tilde{X}_{(0)}$};
		\node [style=none] (26) at (-5, -0.75) {$\tilde{X}_{(1)}$};
		\node [style=none] (27) at (5.25, -0.75) {$\tilde{X}_{(n-1)}$};
		\node [style=none] (28) at (7, -0.75) {$\tilde{X}_{(n)}$};
		\node [style=none] (29) at (-6.5, 0) {$\times$};
		\node [style=none] (30) at (-5, 0) {$\times$};
		\node [style=none] (31) at (-1.5, 0) {$\times$};
		\node [style=none] (32) at (5.25, 0) {$\times$};
		\node [style=none] (33) at (7, 0) {$\times$};
		\node [style=none] (34) at (-3.5, -0.75) {$\dots$};
		\node [style=none] (35) at (1.5, -0.75) {$\dots$};
		\node [style=none] (36) at (-7.75, 2) {};
		\node [style=none] (37) at (-3.5, 2) {};
		\node [style=none] (38) at (-2.75, 2) {};
		\node [style=none] (39) at (2.75, 2) {};
		\node [style=none] (40) at (3.25, 2) {};
		\node [style=none] (41) at (7.75, 2) {};
		\node [style=none] (42) at (0, 3.25) {$\delta^{\mathrm{JE}}(1, \tilde{X}, x)\leq n/4$};
		\node [style=none] (43) at (-6, 3.25) {$\geq n/4$ points};
		\node [style=none] (44) at (6, 3.25) {$\geq n/4$ points};
	\end{pgfonlayer}
	\begin{pgfonlayer}{edgelayer}
		\draw (3.center) to (4.center);
		\draw (14.center) to (15.center);
		\draw (5.center) to (6.center);
		\draw (7.center) to (8.center);
		\draw (9.center) to (10.center);
		\draw (16.center) to (17.center);
		\draw (18.center) to (19.center);
		\draw (20.center) to (21.center);
		\draw [bend left=15] (36.center) to (37.center);
		\draw [bend left=15] (38.center) to (39.center);
		\draw [bend left=15] (40.center) to (41.center);
	\end{pgfonlayer}
\end{tikzpicture}
    \caption{$\delta^{\mathrm{JE}}(1, \tilde{X}, x)$ is bounded by $n/4$ for $-\alpha/4\leq x\leq \alpha/4$ on the event $G$.}\label{fig:simpleexample}
\end{figure*}

In order to give some insight on the general analysis of HSJointExp, and to explain the choice that we suggest for the amplitude $\alpha$ of the noise, we start by discussing the simple setting of \Cref{exa:constant} where $X_i\equiv 0$ and JointExp is known to fail.
We consider uniform noise with distribution $d\mathbb{P}_w(w) = \frac{\one_{[-\alpha, \alpha]}(w)}{2 \alpha} dw$, and HSJointExp returns the output of ExponentialQuantile/JointExp with $m=1$ on the noisy data $\tilde{X}$:
\begin{equation*}
    M := \mathcal{E}_{u_{\text{JE}}}^{(2/\epsilon)}(\tilde{\mathbf{X}}).
\end{equation*}
The true median of the dataset is $0$, and we study  the quadratic risk $\mathbb{E}(M^2)$ of our mechanism.
Note that the classical way of analyzing exponential mechanisms is to use the utility bounds found in \cite{mcsherry2007mechanism}. However, here we do not have the required level of control on the normalization factor. We hence go for a more direct way of controlling the output distribution.
Denoting by $N(x,y) = \sum_{i=1}^n \one_{[x,y)}(0+w_i)$ the number of noisy points falling in the interval $[x,y)$, we define the event
\[G := \big\{N(-\alpha, -\alpha/4) \geq n/4 \big\}\cap \big\{N(\alpha/4, \alpha) \geq n/4\big\}\;.\]
Since $N(-\alpha, -\alpha/4) \stackrel{\mathcal{L}}{=} N(\alpha/4, \alpha)\sim \mathcal{B}(n, 3/8)$, by Hoeffding's inequality, the probability of $G$ is a least $1-2\exp(-n/32)$.
Moreover, on the event $G$, for every $x\in[-\alpha/4, \alpha/4]$ one has $N(-\alpha, x)\geq n/4$ and $N(x, \alpha)\geq n/4$; hence, the minimal number of sample points that need to be changed so as to reach a median equal to $x$ is at most $\delta^{\mathrm{JE}}(1, \tilde{X}, x) =\big| n/2 - N(-1,x) \big|\leq n/4$ (see Figure~\ref{fig:simpleexample}), and $u_{\mathrm{JE}}(\tilde{X}, x)\leq n/8$.
On the other hand, for every $x\notin[-\alpha, \alpha]$, $\delta^{\mathrm{JE}}(1, \tilde{X}, x)=n/2$ and $u_{\mathrm{JE}}(\tilde{X}, x) = n/4$.
Since the density of $M$ at $x\in[-1,1]$ is equal to $\exp\big(-u_{\mathrm{JE}}(\tilde{X}, x)\epsilon/2\big) / \\ \int_{-1}^1 \exp\big(-u_{\mathrm{JE}}(\tilde{X}, t)\epsilon/2\big) dt$, 
\begin{align*}
\P\big(|M|>\alpha \big| G\big) &
\leq \frac{\P\big(|M|>\alpha \big| G\big)}{\P\big(|M|\leq \alpha/4 \big| G\big)}\\
&\leq \frac{2 \times e^{-n\epsilon/8}}{\alpha/2 \times e^{-n\epsilon/16}}  = \frac{4e^{-n\epsilon/16}}{\alpha}\;.
\end{align*}
Therefore,
\begin{align*}
\mathbb{E}\left( M^2 \right) & \leq 1^2\, \big(\P(\bar{G}) + \P\big(|M|>\alpha \big| G\big) \big) +
\alpha^2\,\P\big(|M|\leq \alpha \big| G\big) 
\\& \leq e^{-n/32} + \frac{4e^{-n\epsilon/16}}{\alpha}  + \alpha^2\;.
\end{align*}
Choosing $\alpha = e^{-n\epsilon/48}$ yields \[\mathbb{E}\left( M^2 \right) \leq 5e^{-n\epsilon/24}+e^{-n/32} \;.\]
We conclude that, contrary to JointExp, HSJointExp is here consistent  as soon as $n\epsilon\to \infty$, which is anyway a necessary condition. Besides, the analysis provides a simple and generic way to tune the noise amplitude $\alpha$ as a function of $n$ and $\epsilon$.

\subsection{General Consistency of HSJointExp}

For multiquantile estimation, JoitExp/IS is not endowed with any satisfying statistical utility bounds. We start by proving their consistency in the favorable case of continuous distributions (see \Cref{convThm}). We then leverage this result in order to prove the consistency of HSJointExp on a larger class of distributions.
Indeed, the consistency of HSJointExp is established by making modifications to the density (via the noise) so that we fall into the favorable cases of JointExp. 
\begin{theorem}
\label[theorem]{convThm}
    If $X$ is a random variable with density $\pi_X$ w.r.t. Lebesgue measure that is piecewise continuous and if there exists $\beta>0$ such that $\pi_{X} > 0$ and is continuous on $\cup_{i=1}^{n} [F_{\tilde{X}}^{-1}(p_i)-\beta, F_{\tilde{X}}^{-1}(p_i)+\beta]$, then 
    \begin{equation*}
        \mathbb{P} \left(\|\mathcal{E}_{u_{\text{JE}}}^{(2/\epsilon)}(\mathbf{X}) - F_X^{-1}(\mathbf{p})\|_\infty > \beta \right)= o_n(1) \;.
    \end{equation*}
\end{theorem}
The proof that uses similar techniques as in \Cref{HSJointExpcst} is in \Cref{convThmproof}. The reader can find an expression of the upper bound $o_n(1)$ that does not hide any problem parameter in the proof.
This theorem states that for data distributions with continuous densities, JointExp is consistent and this is to the best of our knowledge the first general result stating the consistency of JointExp.

Back to our problem, the data distribution is not so regular. In particular, we are interested in the case where it contains atoms. Following the method that we propose for HSJointExp, we add some independent and identically distributed noise to the data points.
We decompose the error on the estimation in two terms:
The error measuring the gap between the quantiles of $\mathbb{P}_{\mathbf{X}}$ and the ones of $\mathbb{P}_{\tilde{\mathbf{X}}}$ and the error made by JointExp on the estimation of the quantiles of $\mathbb{P}_{\tilde{\mathbf{X}}}$.
The first term can be controlled by the following general purpose proposition which proof is postponed to \Cref{quantiledeviationproof}.
\begin{proposition}
\label[proposition]{quantiledeviation}
    For any non-increasing $f : \mathbb{R} \to [0, 1]$ such that $\forall t \geq 0, \mathbb{P}(|w| > t) \leq f(t)$, then for every $p \in (0, 1)$, for every $t \geq 0$ such that $1-f(t) > 0$,
    \begin{equation*}
        F_{\tilde{X}}^{-1}\left( p\right)
        \leq F_{X}^{-1}\left( \frac{p}{1 - f(t)}\right) + t \;,
    \end{equation*}
    \begin{equation*}
        \sup_{\delta \in (0, p)} -F_{-X}^{-1}\left( \frac{1-p + \delta}{1 - f(t)}\right) - t
        \leq F_{\tilde{X}}^{-1}\left( p\right) \;.
    \end{equation*}
\end{proposition}
For instance, when applied to some noise with distribution $d\mathbb{P}_w(w) = \frac{\one_{[-\alpha, \alpha]}(w)}{2 \alpha} dw$ with $t=\alpha$ and $f(t) = 0$, if $F_X$ is continuous and strictly increasing on a neighborhood of $F_{X}^{-1}\left( p\right)$, we can say that $|F_{X}^{-1}\left( p\right) - F_{\tilde{X}}^{-1}\left( p\right)| \leq \alpha$.
The second error term can be controlled with \Cref{convThm} assuming that we fall into its hypothesis. By adding some uniform noise in $[-\alpha, \alpha]$, we then obtain the following result:
\begin{theorem}
\label[theorem]{convThmHSJE}
    If the distribution of $X$ is a mixture of a finite number of Diracs in $(a, b)$ and of a random variable $Y$ with a continuous density $\pi_Y$ on $[a, b]$ w.r.t. Lebesgue's measure such that $\pi_Y>0$ on $[a, b] \setminus \mathcal{O}$ where $\mathcal{O}$ is a finite union of intervals and $\pi_Y=0$ on $\mathcal{O}$, then for any precision $\delta$ and Lebesgue-almost-any probability vector $\mathbf{p}$, there exist a noise level $\alpha>0$ such that the $\epsilon$-DP estimator $\mathbf{q}$ based on HSJointExp satisfies
    \begin{equation*}
        \|\mathbf{q} - F_X^{-1}(\mathbf{p})\|_\infty \leq \delta
    \end{equation*}
    with high probability (as $n$ grows).
\end{theorem}
The proof is in \Cref{convThmHSJEproof}.
\Cref{convThmHSJE} states in particular that many distributions that satisfy the hypothesis of \Cref{JointExpFailureTheorem} and on which JointExp is not consistent also satisfy the hypothesis of \Cref{convThmHSJE} and HSJointExp can thus achieve arbitrary levels of precision on them (provided $n$ is large enough).

As highlighted by \Cref{HSJointExpcst}, working on much stricter distribution classes can lead to numerically tractable optimal levels of noise.

\subsection{Privacy Amplification of HSJointExp}

A final property that we would like to explore is the possible amplification of privacy of HSJointExp.
Indeed, adding Laplace or Gaussian noise to bounded quantities is a common way to make them private \cite{dwork2006calibrating}. Furthermore, it is well known that some preprocessing steps (prior to the application of an already private mechanism) increase the provable privacy of the overall mechanism. This is for instance the case with subsampling \cite{DBLP:journals/corr/abs-1807-01647}. Consequently, one would think that adding noise to the data does not only preserve the privacy guarantees of the original mechanism (as stated by \Cref{preprocessionglemma}), but has reasonable chances to make it more private.
In order to evaluate the actual privacy of our mechanism, 
we investigate its privacy loss:
\begin{equation*}
    \mathcal{L}(\mathbf{X}, \mathbf{Y}, \mathbf{q}) :=
    \frac{d\mathbb{P}/d\mathbf{q} \left(\mathcal{E}_{u_{\text{JE}}}^{(2/\epsilon)}(\tilde{\mathbf{X}}) = \mathbf{q}\right)}{d\mathbb{P}/d\mathbf{q} \left(\mathcal{E}_{u_{\text{JE}}}^{(2/\epsilon)}(\tilde{\mathbf{Y}}) = \mathbf{q}\right)}
\end{equation*}
for $\mathbf{Y} \sim \mathbf{X}$ and $\mathbf{q} \in O$ where $d\mathbb{P}/d\mathbf{q} \left(\mathcal{E}_{u_{\text{JE}}}^{(2/\epsilon)}(\tilde{\mathbf{X}}) = \mathbf{q}\right)$ refers to the value of the density of HSJointExp applied to $\mathbf{X}$ at $\mathbf{q}$. 
For a given dataset $\mathbf{X}$, we define $\epsilon_{\text{eff}} := \sup_{\mathbf{X} \sim \mathbf{Y}} \sup_{\mathbf{q}} \log \left(\mathcal{L}(\mathbf{X}, \mathbf{Y}, \mathbf{q}) \right)$ the effective difficulty of distinguishing $\mathbf{X}$ from any of its neighbors.
We always have that $\epsilon_{\text{eff}} \leq \epsilon$ but we would like to measure the difference between the two and its dependence on the noise level.
The theoretical study of such is out of the scope of this article and is left for future work, but we conduct a numerical analysis in \Cref{section5}.

\section{Numerical Results}
\label{section5}

This section presents the behaviors of JointExp, the Inverse Sensitivity mechanism and HSJointExp on synthetic and on on real-world distributions. In particular, \ref{dataset_section} is devoted to the presentation of the distributions of interest. \Cref{numericalperformance} numerically studies the performance of the algorithms on the above-mentioned distributions. And finally, \Cref{numericalprivacy} looks at the possible numerical gain of privacy resulting of the noise addition.

\subsection{Distributions}
\label{dataset_section}

We claimed that HSJointExp has a huge advantage over regular JointExp in the case of distributions with isolated atoms. In order to test it numerically, we propose to do so with synthetic data in the first place. Indeed, in allows us to tune various interesting quantities. For real world distributions, it is harder to identify which ones satisfy the condition of having isolated atoms. We propose to evaluate the performance of the algorithms by identifying a real-wold distribution with the empirical distribution of a real-world dataset. The concentration of this dataset (i.e. how peaked its histogram is) is then the decisive criterion: The more concentrated it is, the more suboptimal JointExp/IS is expected to be compared to the smoothed variants.

\begin{figure*}[t!]
    
      \centering
      \includegraphics[width=\linewidth]{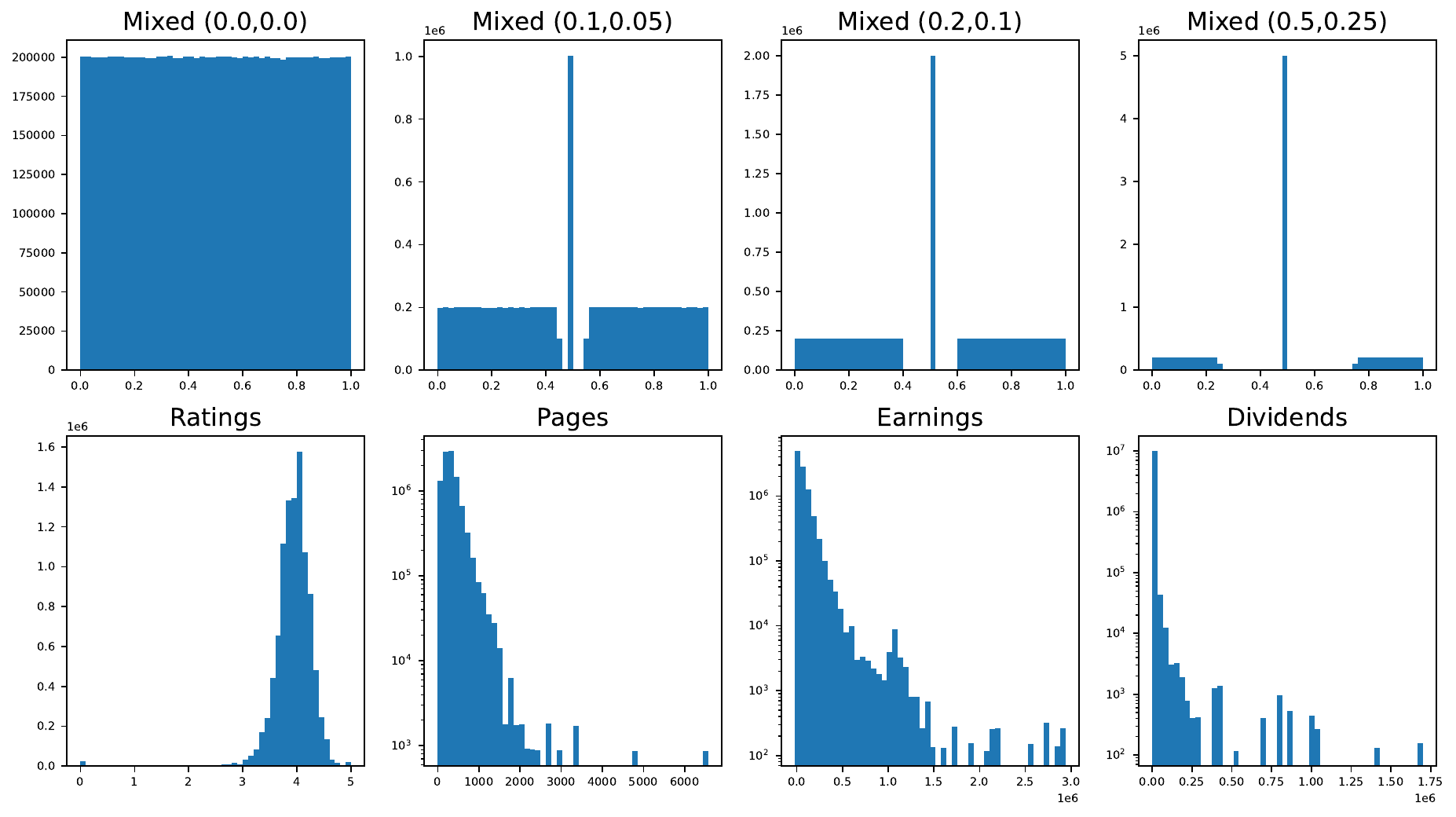}  
    \centering
    Histograms representing $n=10^7$ data points sampled from the original distributions and binned in $50$ bins. Note that for Pages, Earnings and Dividends, the vertical axis is in $\log_{10}$-scale.
    \caption{Distributions used for experiments}
    \label{distributions}
\end{figure*}

\begin{figure*}[t!]
      \centering
      \includegraphics[width=\linewidth]{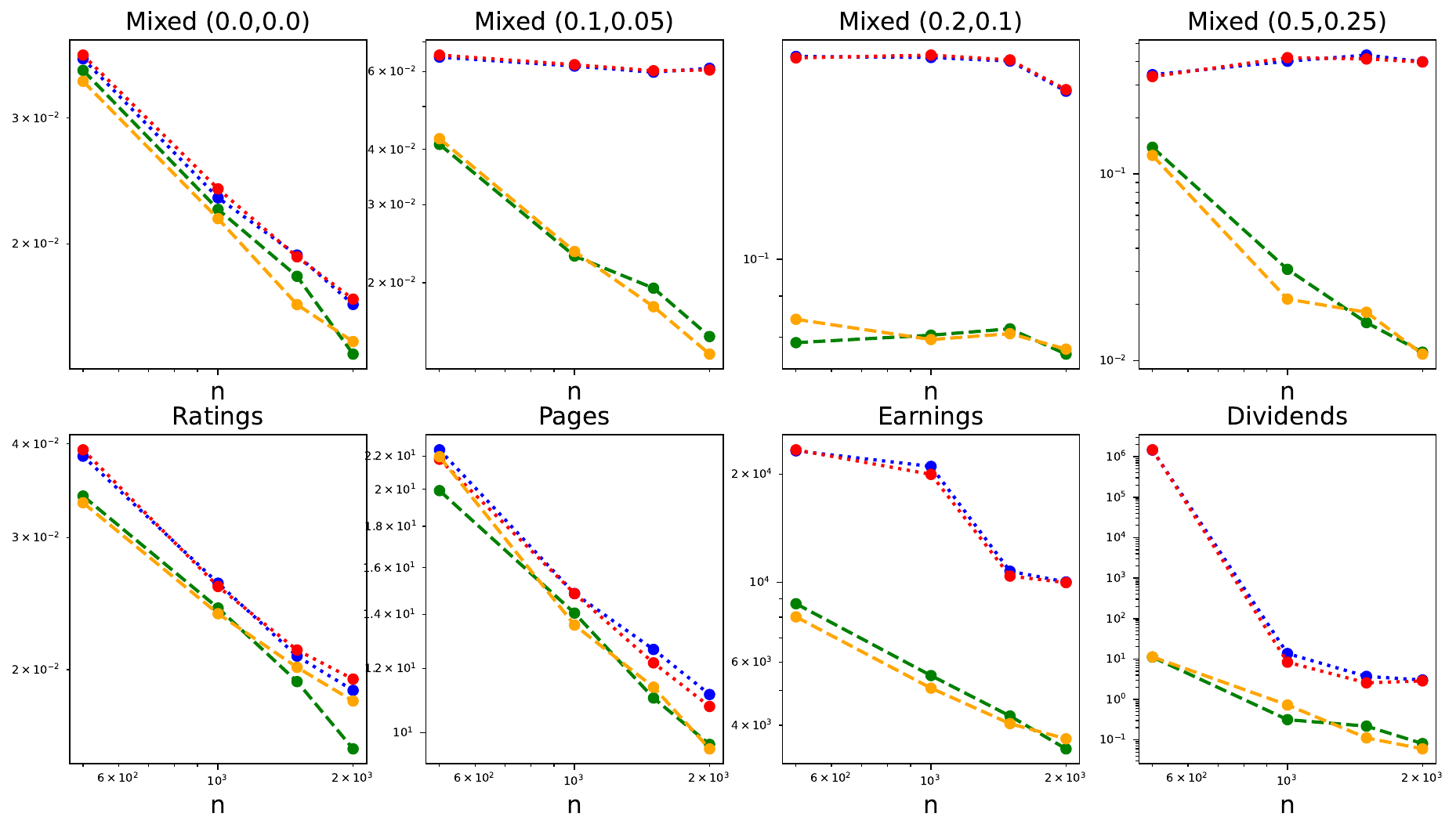}  
      \\
    \centering
    \includegraphics[width=0.9\linewidth]{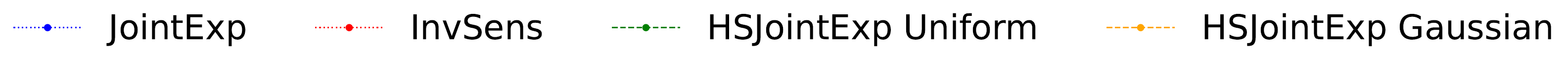}  \\
    \centering

    The vertical axis reads the error $\mathbb{E} \left( \|\hat{\mathbf{q}} - F^{-1}(\mathbf{p})\|_\infty \right)$ where $\mathbf{p} = \left( \frac{1}{m+1}, \dots, \frac{m}{m+1}\right)$ for $m=8$, $\epsilon=1$, $\hat{\mathbf{q}}$ is the private estimator, and $\mathbb{E}$ is estimated by Monte-Carlo averaging over $50$ runs. For HSJointExp Uniform and Gaussian, the optimal noise level on a discretization of $\log_{10}$-resolution $2$ of $[10^{-10}, 10^{4}]$ is selected. Note that both axis are in $\log_{10}$-scale.
    \caption{Error of the estimators as a function of $n$}
    \label{consistency}
\end{figure*}

\paragraph{Mixed distributions (synthetic).}
For $p \in [0, 1]$ and $\delta \in [0, 1/2]$, we define the \emph{Mixed} distribution of parameters $(p, \delta)$ as the distribution with support in $[0, 1/2-\delta] \cup \{1/2\} \cup [1/2+\delta, 1]$ such that if a random variable $X$ follows this distribution, we have $\mathbb{P}(X = 1/2) = p$, $\mathbb{P}(X \in [0, 1/2-\delta]) = \mathbb{P}(X \in [1/2+\delta, 1])$, and conditionally to the event $(X \in [0, 1/2-\delta])$ or to the event $(X \in [1/2+\delta, 1])$, $X$ is uniform. In particular, the mixed distribution of parameters $(0, 0)$ is the uniform distribution on $[0, 1]$. In order to better visualize such distributions, sampled histograms are represented in \Cref{distributions}.
The parameters $\epsilon$ and $\delta$ allow tuning, respectively, the probability of the atom and its isolation. The bigger they are, the more HSJointExp is expected to outperform the non-smoothed variants.

\paragraph{Pages and Ratings (real-world).} The distributions that we call \emph{Pages} and \emph{Ratings} correspond to the empirical distributions of a collection of ratings and of number of pages of books from the Goodreads-Books dataset \cite{goodreadsbooks}. Gillenwater et al. \cite{gillenwater2021differentially} used the same datasets as numerical evidences of the performance of JointExp for estimating empirical quantiles. Again, sampled histograms are represented in \Cref{distributions}. The distributions look relatively smooth (i.e. not too peaked and with a relatively small support), and as a result, we can expect the gap between JointExp/IS and HSJointExp to be negligible. 

\paragraph{Earnings and Dividends (real-world).}
The distributions that we call \emph{Earnings} and \emph{Dividends} correspond respectively to the personal incomes and personal incomes from dividends categories of the US 2021 Census \cite{census2021}. Again, sampled histograms are represented in \Cref{distributions}. We can notice that contrary to the previous two real-world distributions, these two are much more concentrated. For Earnings, the concentration is due to the existence of categories of extremely high revenues. As a consequence, the support of the distribution is necessary big, and the algorithms that seek for privately estimating the quantiles have little information about the localization of the data points. On the other hand, the vast majority of people declare revenues inferior to $500000$ dollars, resulting in the high concentration of the distribution close to $0$. For Dividends, the support is smaller, but since a big part of the population simply does not have any revenues from dividends, the distribution shows an accumulation point at $0$. With both distributions, we expect the smoothing operation to vastly improve the performance of JointExp/IS.

\subsection{Numerical Performance}
\label{numericalperformance}

\Cref{consistency} and \Cref{noiselvl} Compare the performance of JointExp, the Inverse Sensitivity mechanism and two variants of HSJointExp with uniform and Gaussian noise structure respectively on the distributions presented in \Cref{distributions}.

\paragraph{Complements on HSJointExp Uniform and Gaussian.}

The mechanism that we call \emph{HSJointExp Uniform} is the application of JointExp post addition of centered uniform noise. If $[a, b]$ was our estimate of the support of the distribution, we apply JointExp on $[a-\sigma\sqrt{3}, b+\sigma\sqrt{3}]$ where $\sigma$ is the standard deviation of the noise. In  \emph{HSJointExp Gaussian}, the centered uniform noise is replaced by centered Gaussian noise. The support of the resulting distribution is now infinite, and the projection step is therefore mandatory. We chose to project the data points in $[a-5\sigma, b+5\sigma]$  where $\sigma$ is the standard deviation of the noise in order to make sure that most of the points will remain untouched by the projection step.

\begin{figure*}[h!]
    
      \centering
      \includegraphics[width=\linewidth]{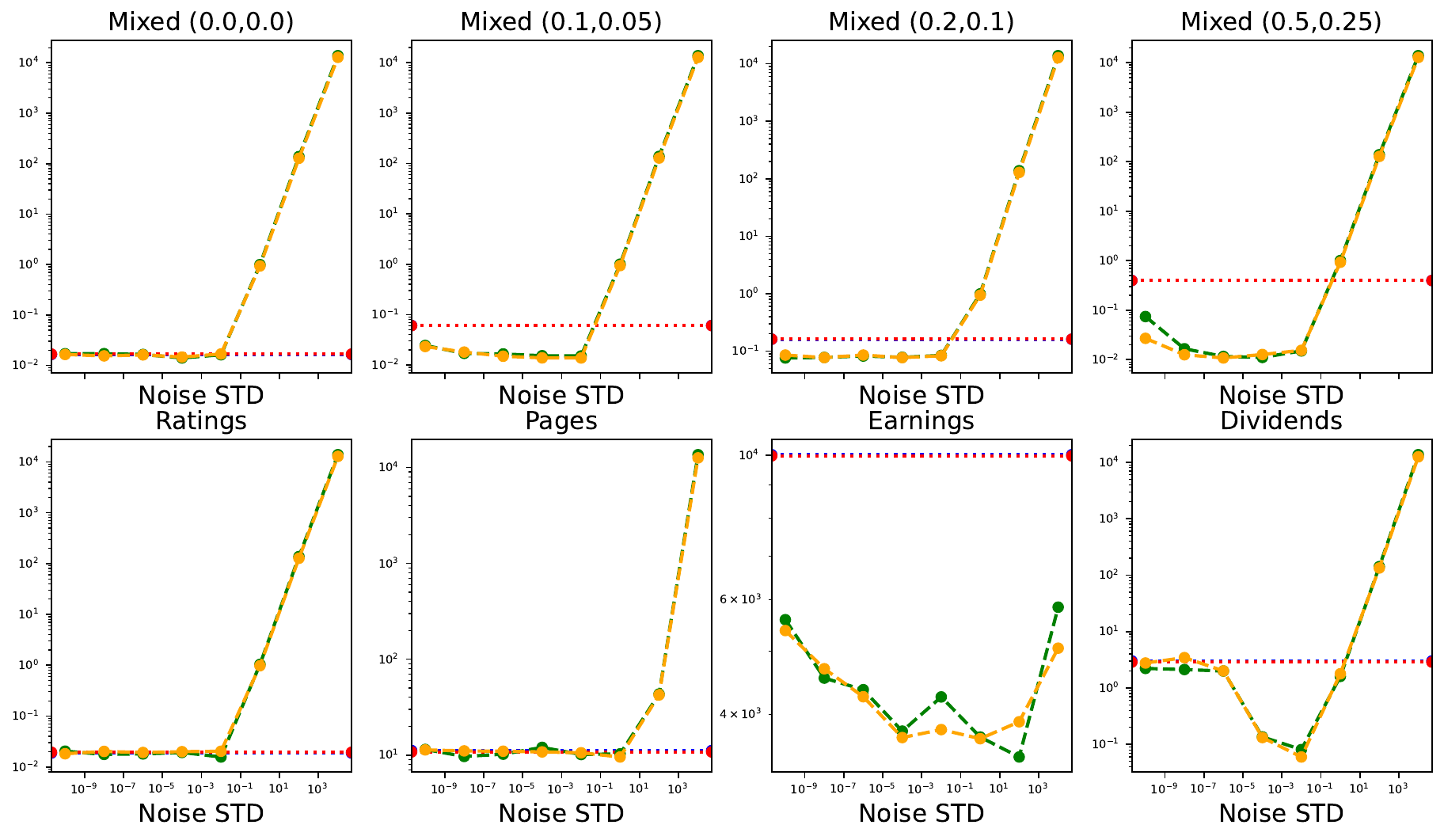}  
      \\

    \centering
    \includegraphics[width=0.9\linewidth]{legend2.pdf}  \\
   The vertical axis reads the error $\mathbb{E} \left( \|\hat{\mathbf{q}} - F^{-1}(\mathbf{p})\|_\infty \right)$ where $n=2000$, $\mathbf{p} = \left( \frac{1}{m+1}, \dots, \frac{m}{m+1}\right)$ for $m=8$, $\epsilon=1$, $\hat{\mathbf{q}}$ is the private estimator, and $\mathbb{E}$ is estimated by Monte-Carlo averaging over $50$ runs. For HSJointExp Uniform and Gaussian, the optimal noise level on a discretization of $\log_{10}$-resolution $2$ of $[10^{-10}, 10^{4}]$ is selected. Note that both axis are in $\log_{10}$-scale. The horizontal axis reads the standard deviation of the smoothing noise. JointExp and the Inverse Sensitivity mechanism are represented by horizontal bars, since they do net depend on the noise level.
    \caption{Dependence on the smoothing level}
    \label{noiselvl}
    \end{figure*}

\paragraph{Analyzing the results of \Cref{consistency}.}

The first important fact to notice is the similar performance of JointExp and the Inverse Sensitivity mechanism, confirming the theoretical results of \Cref{sec:link}. The second is the similar performance of HSJointExp Uniform and Gaussian, showing that the structure of the noise, given that it is regular enough, is not of critical importance. Finally, and probably the most important, we can compare the performance of JointExp/IS and of HSJointExp. On Mixed$(0,0)$ (i.e. the uniform distribution on $[0, 1]$), Ratings and Pages, the two algorithms perform identically. This is what we expected given the smoothness of the distributions. On more concentrated distributions like Earnings and Dividends on the other hand, we see that HSJointExp vastly improves the performance of JointExp, sometimes by multiple orders of magnitude. Finally, Mixed$(0.1,0.05)$, Mixed$(0.2,0.1)$ and Mixed$(0.5,0.25)$ demonstrate that the more isolated and probable the atoms of the distribution are, the more suboptimal JointExp is compared to the smoothed variants.

\paragraph{Analyzing the results of \Cref{noiselvl}.}
\Cref{noiselvl} shows the same results as \Cref{consistency} but with an emphasis on the dependence on the noise level. For instance, we can see that when the smoothing operation allows for better performance, it is often the case for a large range of smoothing levels. Finally, we can numerically observe two limit behaviors that are quite intuitive : When the noise level tends to $0$, HSJointExp performs as JointExp. Indeed, in this case, the smoothing trick has almost no effect on the distribution. When the noise level tends to $+\infty$ on the other hand, the performance of HSJointExp is terrible. This is also quite intuitive, since the smoothed distribution has lost almost all correlation with the original distribution. For all these reasons, we recommend tuning the noise as in the extreme case of the Dirac (see \Cref{HSJointExpcst}) since this value is small enough to not fall in the regime where the performance are degraded by the smoothing, but it still greatly improves the performance on degenerated distributions.

\subsection{Privacy Amplification}
\label{numericalprivacy}

In \Cref{effective_epsilon} we numerically estimate $\epsilon_{\text{eff}}$ in the following setup: For each of the datasets (noted $\mathbf{X}$), we estimate the median using HSJointExp with Laplace noise tuned with $\epsilon=1$. We estimate $\mathcal{L}(\mathbf{X}, \mathbf{Y}, \mathbf{q})$ for any $\mathbf{Y} \sim \mathbf{X}$ by discretizing the search space of $\mathbf{Y}$ and by Monte Carlo averaging to integrate with respect to the noise. 

\begin{figure*}[ht!]
    \begin{subfigure}{.24\textwidth}
      \centering
      \includegraphics[width=\linewidth]{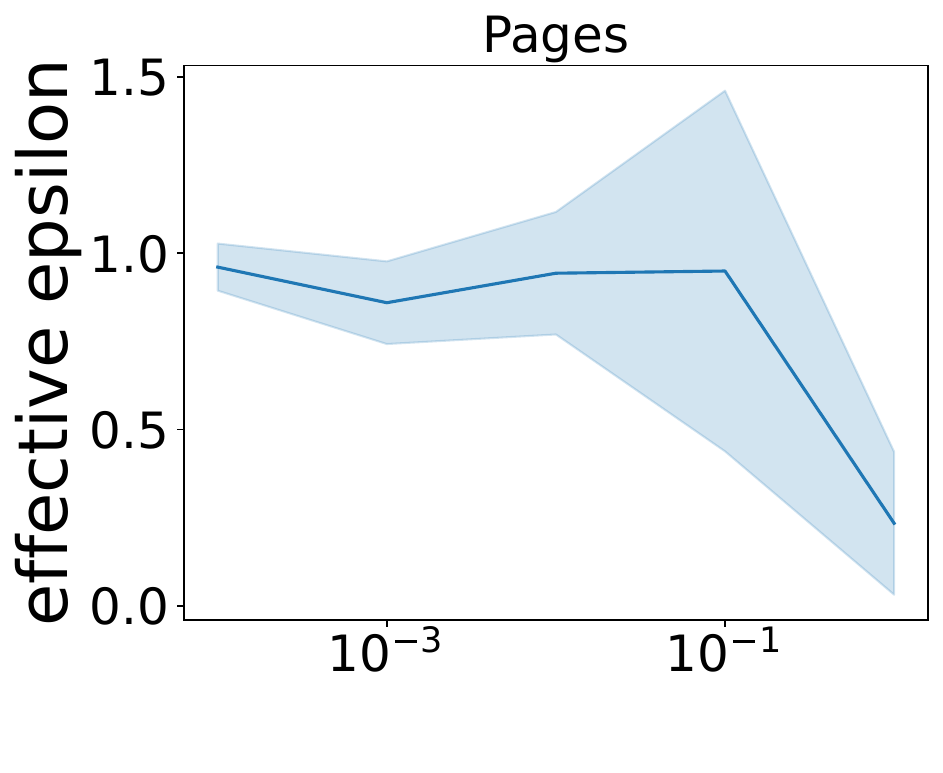}  
      \label{fig:sub-first}
    \end{subfigure}
    \begin{subfigure}{.24\textwidth}
      \centering
      \includegraphics[width=\linewidth]{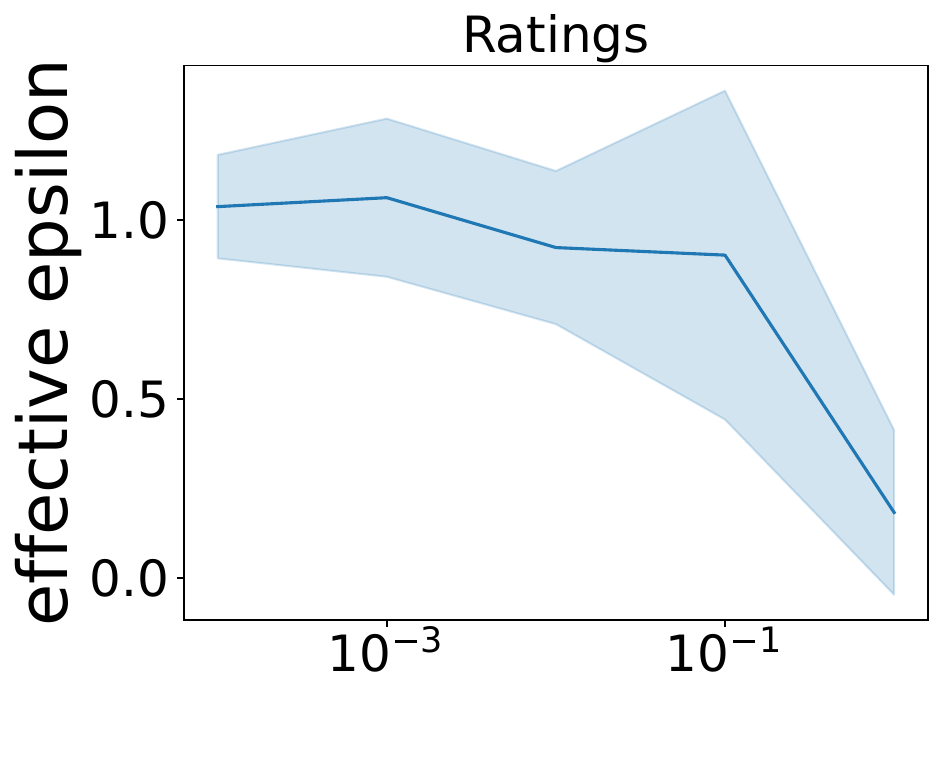}  
      \label{fig:sub-second}
    \end{subfigure} 
    \begin{subfigure}{.24\textwidth}
        \centering
        \includegraphics[width=\linewidth]{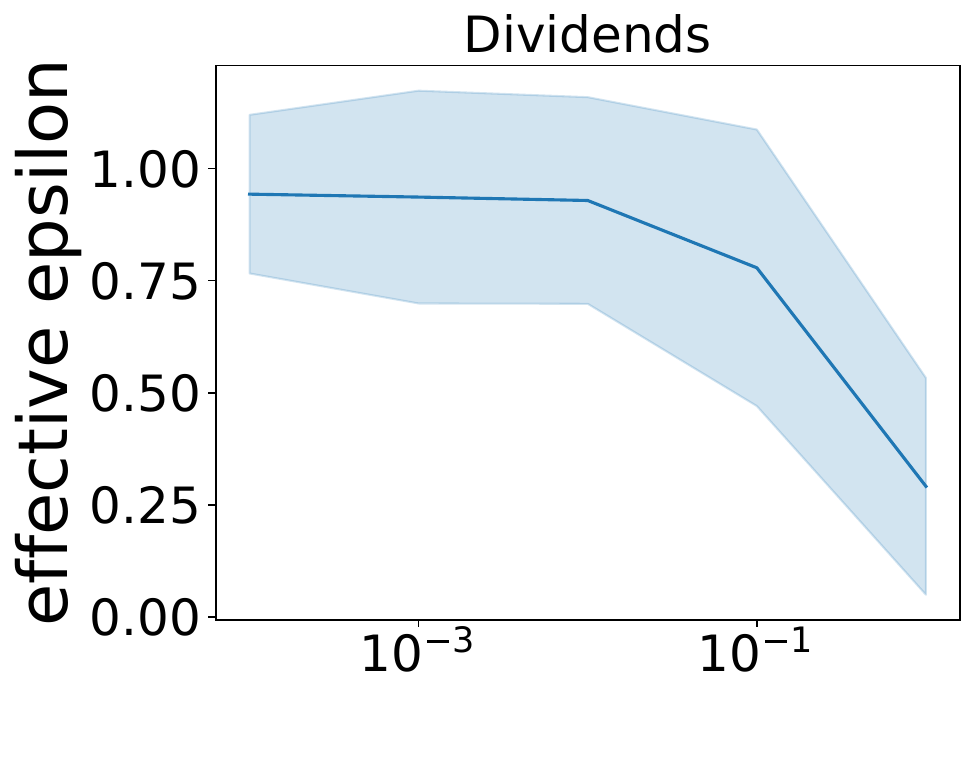}  
        \label{fig:sub-second}
      \end{subfigure}
    \begin{subfigure}{.24\textwidth}
        \centering
        \includegraphics[width=\linewidth]{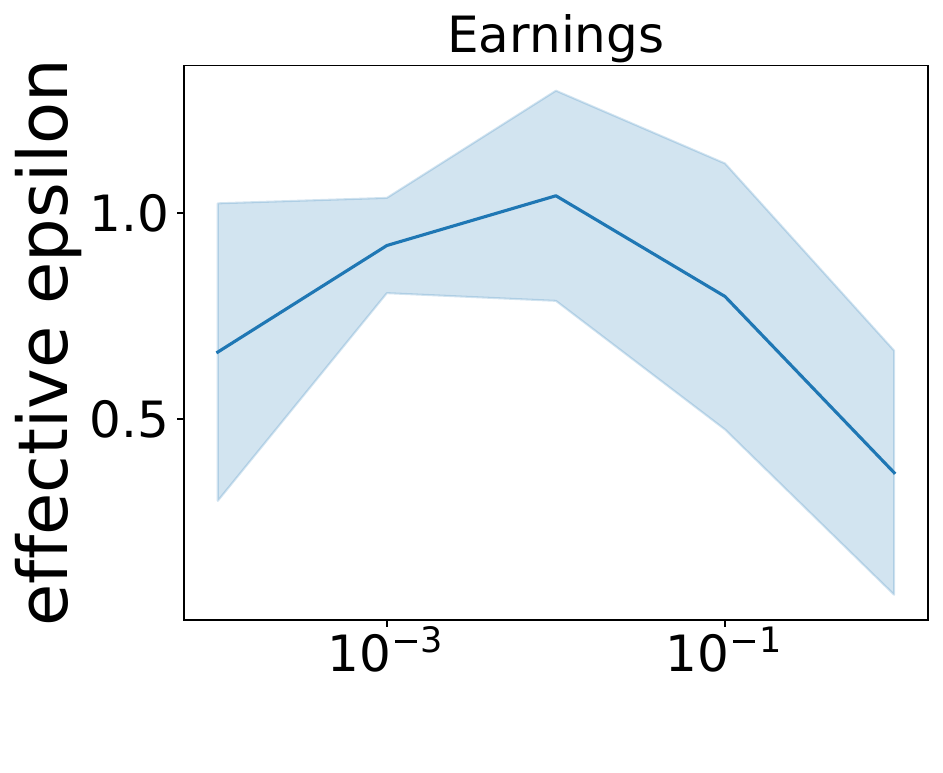}  
        \label{fig:sub-second}
    \end{subfigure}
    \centering
    The horizontal axis represents the standard deviation of the noise divided by the length of the support of the distribution. $\epsilon = 1$.
    \caption{Evolution of $\epsilon_{\text{eff}}$ for the median estimation}
    \label{effective_epsilon}
\end{figure*}

The variance of the resulting $\epsilon_{\text{eff}}$ is high, but we can see two regimes: For low values of the noise, the privacy of the mechanism is unchanged. For high values of noise, on the other hand, $\epsilon_{\text{eff}} < \epsilon$ and differentiating the datasets from their neighbors is harder.

By crossing the results with \Cref{noiselvl} however, it seems that the privacy amplification only occurs for values of the noise for which the utility of HSJointExp is already degraded compared to regular JointExp.

\section{Conclusion}

We highlight the connections between the general inverse sensitivity mechanism applied to the private estimation of multiple empirical quantiles and the recently published ad-hoc algorithm JointExp \cite{gillenwater2021differentially}. We prove the consistency of this algorithm when used as a statistical estimator of the statistical quantiles of the underlying distribution for smooth enough distributions. These results are key to the understanding of JointExp, which wasn't endowed with any theoretical utility results before despite a excellent numerical behavior. Furthermore, we demonstrate that isolated atoms in the distribution cause JointExp to be inconsistent, and propose a numerically tractable fix that solves this issue. The numerical experiments on both real-world and synthetic distributions backup the theoretical claims of this article, and support our suggestion to use our variant instead of JointExp in practice.

Very recently, a new approach was proposed by Kaplan et al. in \cite{kaplan2022differentially}: they present an algorithm that experimentally beats JointExp when the number of empirical quantiles of a dataset is high. Determining under what conditions and in what circumstances this empirical gap reflects in terms of statistical utility is left as future work. 

\section*{Acknowledgments}
Aurélien Garivier acknowledges the support of the Project IDEXLYON of the University of Lyon, in the framework of the Programme Investissements d'Avenir (ANR-16-IDEX-0005), and Chaire SeqALO (ANR-20-CHIA-0020-01). This project was supported in part by the AllegroAssai ANR project ANR-19-CHIA-0009.


\newpage

\bibliographystyle{plain}
\bibliography{biblio}

\newpage
\appendix

\section{Sampling from the inverse sensitivity mechanism}
\label{sampleis}

In this subsection, we explain how to sample exactly from the inverse sensitivity mechanism for multiple quantiles in polynomial time and memory. It is essentially an adaptation from the JointExp algorithm, and hence we will use the same notations when possible. For simplicity, $\mathbf{X}$ is assumed to be sorted.

The sampling density of $\mathcal{E}_{\tilde{u}_{\text{IS}}}^{(2/\epsilon)}(\mathbf{X})$ is constant on sets $\big([X_{i_1}, X_{i_1 + 1}) \times \ldots
\times \dots \times [X_{i_{m}}, X_{i_m +1}) \big) \cap [a, b]^m \!\!\!\!\!\nearrow$ for $\mathbf{i}=(i_1, \dots, i_m) \in O'$ where
\begin{equation*}
    O'= \left\{\mathbf{i} \in \left\{0, \dots, n\right\}^m, 0 \leq i_1\leq \dots \leq i_m \leq m \right\} \;.
\end{equation*}
Hence, a finite sampling algorithm for $\mathcal{E}_{u_{\text{IS}}}^{(2/\epsilon)}(\mathbf{X})$ is to:
\begin{itemize}
    \item sample $\mathbf{i}=(i_1, \dots, i_m) \in O'$ under $\mathbb{P}_{O'}$;
    \item sample $q_j'$ uniformly in $[X_{i_j}, X_{i_j + 1})$, independently for all $j$ in $\{1 \dots m\}$;
    \item output $(q_j')_{j \in \{1 \dots m\}}$ sorted by increasing order;
\end{itemize}
with the probability $\mathbb{P}_{O'}$  defined on $O'$ as
\begin{equation}
    \label{probfact}
    \mathbb{P}_{O'}(\mathbf{i}) \propto \frac{1}{\gamma(\mathbf{i})}
    \Pi_{j=1}^{m+1} \phi(i_{j-1}, i_{j}, j) \Pi_{j=1}^{m} \tau(i_j)
\end{equation}
where, if we denote by $\text{count}_{\mathbf{i}}(i)$ the number of occurrences of integer $i$ in the ordered tuple $\mathbf{i}$,
\begin{equation*}
    \forall \mathbf{i} \in O', \gamma(\mathbf{i}) = \Pi_{i=0}^m \text{count}_{\mathbf{i}}(i)!\;,
\end{equation*}
\begin{equation*}
    \forall i \in \{0, \dots, m\}, \tau(i) = X_{i+1} - X_i \;,
\end{equation*}
and for $0\leq i,i' \leq m$ and $1\leq j \leq m+1$,
\begin{equation*}
    \begin{aligned}
        \phi(i, i', j) = \begin{cases} 
            0, & \mbox{if } i'<i \\ 
            e^{-\frac{\epsilon}{2} \left( \frac{1}{2} |\hat{\delta}(i, i', m+1)| \right)}, & \mbox{if } j=m+1 \\
            e^{-\frac{\epsilon}{2} \left( \frac{1}{2} |\hat{\delta}(i, i', j)| \right)+ \one_{\mathbb{R}_+}(\hat{\delta}(i, i', j))} , & \mbox{otherwise }
        \end{cases}
    \end{aligned}
\end{equation*}
with $\hat{\delta}(i, i', j) = i' - i - (\lceil n p_{j} \rceil - \lceil n p_{j-1} \rceil)$.

Since $O'$ has a finite cardinality bounded by $(n+1)^m$, it is possible to compute the probability of all the elements in that space and to sample this way. However, the fact that this complexity is exponential in $m$ makes it unusable in practice.  \cite{gillenwater2021differentially} present an algorithm that allows to sample from any distribution that factorizes in an analog form of \eqref{probfact} that has a complexity (both in time and space) of $O(n^2 m + m^2 n)$. Furthermore, if the function $\phi(i, i', j)$ can be rewritten as $\phi'(i'-i, j)$ (which is the case in our problem), the complexity becomes $O(mn\log n + m^2n)$. Overall, in order to sample efficiently from the inverse sensitivity mechanism, one can use Algorithm 1 proposed by  \cite{gillenwater2021differentially} by taking great care of using a sensitivity of $1$ (instead of $2$) and by replacing the function $\phi$ by the one used in this article.

\section{Inverse sensitivity smoothing \cite{asi2020near}}
\label{ISSmoothing}
\subsection{General principle}

When the output space $O$ is a subset of a Euclidean space, the theory of Inverse Sensitivity comes with a smoothing operation  with some parameter $\rho > 0$. The utility function $u_{\text{IS}}$ can be replaced with 
\begin{equation*}
    u_{\text{IS}}^\rho(\mathbf{X}, o) = \sup_{o' \in O: \|o-o'\|_2\leq \rho} u_{\text{IS}}(\mathbf{X}, o') \;.
\end{equation*}
It is easy to see that this new utility function has sensitivity $\Delta u_{\text{IS}}^\rho = 1$. 
Contrary to the non-smoothed inverse sensitivity mechanism which only comes with guarantees for finite output spaces $O$, this smoothed version gives more general results that only rely on a few mathematical tools. Using the notion of \emph{modulus of continuity} of the target function $\mathcal{Q}$, defined as
\begin{equation*}
    \omega_{\mathcal{Q}}(\mathbf{X},k) = \sup_{\mathbf{X}' \in \mathfrak{X}^n: d(\mathbf{X}, \mathbf{X}') \leq k}\Big\{ \|\mathcal{Q}(\mathbf{X}) - \mathcal{Q}(\mathbf{X}')\|_2\Big\} \;,
\end{equation*}
and the corresponding image 
\begin{equation*}
    W_{\mathcal{Q}}(\mathbf{X}, k) =\{ \mathcal{Q}(\mathbf{X}')- \mathcal{Q}(\mathbf{X}): d(\mathbf{X},\mathbf{X}') \leq k\} \;,
\end{equation*}
one can bound the estimation error \cite{asi2021nips} assuming that $\text{diam}_2(\mathcal{Q}(\mathfrak{X}^n)) \leq D$: for $1 \leq k \leq n$
\begin{equation*}
\begin{aligned}
    \mathbb{P}\Big(\| \mathcal{E}_{u_{\text{IS}}^\rho}^{(2/\epsilon)}(\mathbf{X}) - \mathcal{Q}(\mathbf{X})\|_2 &\geq \omega_{\mathcal{Q}}(\mathbf{X},k) + \rho\Big)\\
    &\leq e^{-k\epsilon/2}\left(\tfrac{D}{\rho}\right)^m
    \end{aligned}
\end{equation*}
with $m$ the ambient space dimension (i.e., $O \subset \mathbb{R}^m$).
The original authors consider a smoothing parameter $\rho = 1/n^r$ for some $r>0$ and 
$k$ of the order of $(4 r m \log n) / \epsilon$,
which yields an estimation error bounded with high probability by the modulus of continuity: With high probability
\begin{equation}
    \label{highprobbound}
    \|\mathcal{E}_{u_{\text{IS}}^{1/n^r}}^{(2/\epsilon)}(\mathbf{X}) - \mathbf{X}\|_2 \leq O\left(\omega_{\mathcal{Q}}(\mathbf{X}, (4 r m \log n)/\epsilon) + 1/n^r\right) \;.
\end{equation}
This theory also provides an optimality result. Under the rather strong hypothesis that there exists a uniform $c>0$ such that for all $1 \leq k \leq n$ and $\mathbf{X} \in \mathfrak{X}^n$,
\begin{equation}
    \label{uniformcond}
    W_{\mathcal{Q}}(\mathbf{X}, k) \supseteq c . \omega_{\mathcal{Q}}(\mathbf{X}, k) .  \mathbb{B}_2^m
\end{equation} 
where $\mathbb{B}_2^m$ is the $l_2$ ball in $\mathbb{R}^m$,
the best $\epsilon$-DP algorithm roughly behaves the same way in a local minimax sense up to a logarithmic factor \cite{asi2021nips}:

\begin{equation*}
    \begin{aligned}
    \inf_{\mathcal{A} \in \mathcal{A}_\epsilon}
    \sup_{\mathbf{X}': d(\mathbf{X}, \mathbf{X}')\leq m/\epsilon}
    &\mathbb{E}(\|\mathcal{A}(\mathbf{X}') - \mathcal{Q}(\mathbf{X}')\|_2)\\ 
    &\geq \Omega(\omega_{\mathcal{Q}}(\mathbf{X}, m / \epsilon))
    \end{aligned}
\end{equation*}
where $\mathcal{A}_\epsilon$ is the class of $\epsilon$-DP algorithms.

\subsection{For the multiquantile problem}

\paragraph{The problem of sampling:}

The hypothesis of \Cref{mainth} restrict our ability to compute the inverse sensitivity of quantile candidates without collisions and that do not overlap with any of the data intervals. Computing the supremum in the definition of the smoothed inverse sensitivity would not only require to handle those cases, but also to have an algorithm faster than looking at all the possibilities. We did not manage to overcome that difficulty hence the reason for our heuristic smoothing (see \Cref{hssmoothing}).

\paragraph{Behavior of the modulus of continuity:}
The modulus of continuity measures the maximal variation of a function on a ball for Hamming distance $k$. Here we derive a majoration for the multiquantile problem.
Assuming that $\mathbf{X} \in \mathfrak{X}^n$ is sorted, by moving a single point ($k=1$) of $\mathbf{X}$, two behaviors can happen:
Either it exactly matches one of the quantiles, and the corresponding estimate can then vary continuously in the interval between the data point below and the data point above; or it did not match any quantile, but it can still shift the entire ordered statistics by one data point. This bounds the values of the function $\mathcal{Q}$ at Hamming distance 1 of $\mathbf{X}$ (i.e., $W_{\mathcal{Q}}(\mathbf{X}, 1)$):
\begin{equation*}
    \begin{aligned}
    W_{\mathcal{Q}}(\mathbf{X}, 1) + \mathcal{Q}(\mathbf{X})\subseteq \bigcup_{i=1}^m 
    &\Big\{ \mathbf{X}_{[\lceil n p_1\rceil-1:\lceil n p_1\rceil+1]} \\ 
    &\times  \dots \\ 
    &\times \mathbf{X}_{[\lceil n p_{i-1}\rceil-1:\lceil n p_{i-1}\rceil+1]} \\
    &\times [X_{\lceil n p_{i}\rceil-1}, X_{\lceil n p_{i}\rceil+1}] \\
    &\times \mathbf{X}_{[\lceil n p_{i+1}\rceil-1:\lceil n p_{i+1}\rceil+1]} \\ 
    &\times  \dots \\
    &\times \mathbf{X}_{[\lceil n p_{m}\rceil-1:\lceil n p_{m}\rceil+1]} \Big\}
    \end{aligned}
\end{equation*}
where we use the notation from computer science $\mathbf{X}_{[i:j]} = (X_i, \dots, X_j)$.
When $k\geq 1$ points are moving, the same mechanics arise where the ordered statistics is shifted by at most $k$ points and where we have at most $k$ degrees of freedom. This yields an analog majoration with $k$ intervals in each cartesian product:
\begin{equation}
    \label{majW}
    \begin{aligned}
    W_{\mathcal{Q}}(\mathbf{X}, k) + \mathcal{Q}(\mathbf{X})\subseteq \bigcup_{1 \leq i_1 \leq \dots \leq i_k \leq m } 
    \bigotimes_{j=1}^{m} \mathcal{I}(\mathbf{X}, j, \mathbf{i}, k) \;,
    \end{aligned}
\end{equation}
where, if we note $\mathbf{i} = (i_1, \dots, i_k)$,
\begin{equation*}
    \begin{aligned}
        \mathcal{I}(\mathbf{X}, j, \mathbf{i}, k) = \begin{cases} [X_{\lceil n p_{j}\rceil-k}, X_{\lceil n p_{j}\rceil+k}], & \mbox{if } j \in \mathbf{i} \\ \mathbf{X}_{[\lceil n p_{j}\rceil-k:\lceil n p_{j}\rceil+k]}, & \mbox{if } j \notin \mathbf{i} \end{cases}\;.
    \end{aligned}
\end{equation*}
Using~\eqref{majW}, the modulus of continuity is bounded as
\begin{equation}
    \label{majw}
    \omega_{\mathcal{Q}}(\mathbf{X}, k) \leq \sqrt{\sum_{i=1}^m \big( X_{\lceil n p_{i}\rceil+k} - X_{\lceil n p_{i}\rceil-k} \big)^2}\;.
\end{equation}
Hence, when  the dataset $\mathbf{X}$ has many points close to its empirical quantiles, the modulus of continuity $\omega_{\mathcal{Q}}(\mathbf{X}, k)$ should not grow too fast in $k$.

\paragraph{Convergence with high probability}:
The concentration bound (i.e., Equation~\eqref{highprobbound}) along with the upper bound on the modulus of continuity (i.e., Equation~\eqref{majw}) gives that, with high probability
\begin{equation*}
    \label{concentrationquantile}
    \|\mathcal{E}_{u_{\text{IS}}^{1/n^r}}^{(2/\epsilon)}(\mathbf{X}) - O(\mathbf{X})\|_2 \leq O\left(\sqrt{\sum_{i=1}^m \big( X_{\lceil n p_{i}\rceil+\delta } - X_{\lceil n p_{i}\rceil-\delta } \big)^2}  + \frac{1}{n^r}\right) \;,
\end{equation*}
where $\delta = \lceil(4 r m \log n) / \epsilon \rceil$. Hence, whenever the dataset has many points in the neighborhood of the empirical quantiles, we can expect the smoothed inverse sensitivity mechanism for multiquantile (i.e., $\mathcal{E}_{u_{\text{IS}}^{1/n^r}}^{(2/\epsilon)}(\mathbf{X})$) to perform well. This will typically be the case when the data distribution has a Dirac on a quantile whose mass accounts for more than $\frac{4 r m \log n}{n \epsilon}$ or when the data distribution has a density that is strictly positive on a neighborhood of its quantiles and $\delta$ is not too large. 

\paragraph{Local minimax bound} 

In \eqref{majW}, the inner term $\bigotimes_{j=1}^{m} \mathcal{I}(\mathbf{X}, j, \mathbf{i}, k)$ has null Lebesgue measure if $k < m$. Indeed, in this case, at least one factor in the product is countable. As a consequence, denoting $\lambda$ the Lebesgue measure,
\begin{equation*}
    k < m \Rightarrow \lambda(W_{\mathcal{Q}}(\mathbf{X}, k))=0.
\end{equation*}
 Hence, $W_{\mathcal{Q}}$ does not satisfy the uniform condition expressed in~\eqref{uniformcond} and the local minimax bound falls. The only setup where this bound holds is in the case of single quantile estimation (i.e., $m=1$). As a consequence, we may not expect the optimality of the smoothed inverse sensitivity mechanism for multiquantile in the class of $\epsilon$-DP algorithm. We included this negative result for the sake of exploring the theory of inverse sensitivity as a whole. It shows that even if a sampling procedure for the smoothed inverse sensitivity mechanism for multiquantile was to be found, finding an optimal sampling algorithm for private multiple quantiles is still an open question.

\section{Omitted proofs}
\label{proofs}

\subsection{Proof of \Cref{mainth}}
\label{mainthproof}

    If $\mathbf{Y} \in Q^{-1}(\mathbf{q})$ then:
    \begin{itemize}
    \item Each "bin" has the right number of points:
         $\delta(i,\mathbf{Y},\mathbf{q}) = 0$,
         $i \in \left\{2 \dots m+1\right\}$,
         and \\$\delta_{\texttt{closed}}(1,\mathbf{Y},\mathbf{q}) = 0$.
    \item Every point of $\mathbf{q}$ appears in $\mathbf{Y}$: $\mathbf{q} \subseteq \mathbf{Y}$.
    \end{itemize}
    Then we can understand the modifications that have to be made to $\mathbf{X}$ in order to obtain a $\mathbf{Y} \in Q^{-1}(\mathbf{q})$. 
    For the first condition, some points have to be moved from bins in excess to bins in deficit. This procedure accounts for $\sum_{i=2}^{m+1} \delta(i,\mathbf{X},\mathbf{q})_+ + \delta_{\text{closed}}(1,\mathbf{X},\mathbf{q})_+$ operations which can be reformulated as $\frac{1}{2} \sum_{i=2}^{m+1} \left| \delta(i,\mathbf{X},\mathbf{q})\right| + \frac{1}{2} |\delta_{\text{closed}}(1,\mathbf{X},\mathbf{q})|$.
    For the second condition, we have to make sure that for all $i$, $q_i$ belongs to the dataset.
    For a bin in strict deficit, at least a point has to be added to it due to the first condition. Hence, we can make sure to add the associated quantile at no extra cost. 
    For a bin in excess on the other hand, since by hypothesis $\mathbf{q} \cap \mathbf{X} = \emptyset$, a point in the bin will have to be replaced by the associated quantile at an extra cost of $1$. In the end, we find the desired result.

\subsection{Proof of \Cref{JointExpFailureTheorem}}
\label{JointExpFailureTheoremproof}

    Since $\mathbb{P}_{X}(\{q\})>0$ and $\mathbb{P}_{X}(I \setminus \{q\})=0$, there exists a nonempty interval $A$ of $[0, 1]$ such that $\{q\} = F_{X}^{-1}(A)$ with $\lambda(A)\geq \mathbb{P}_{X}(\{q\})$, $\lambda$ referring to Lebesgue measure. Let us prove that any $\mathbf{p}$ with at least one component in $A$ satisfies \eqref{eqpb}. For this, assume that $\mathbf{p}$ has its $i^\text{th}$ entry $p_i$ in $A$.
    Due to the structure of $\mathbb{P}_{X}$, $\mathbb{P}_{X}(\mathbf{X} \cap (I \setminus \{q\}) \neq \emptyset)=0$, hence almost surely it holds that for every $j$ we have either $|X_j - q| \geq \eta > 0$ 
    or $|X_j - q| = 0$. 
    Remember that the output density is a mixture of uniforms on the sets $\big([X_{i_1}, X_{i_1 + 1}) \times \ldots
\times \dots \times [X_{i_{m}}, X_{i_m +1}) \big) \cap [a, b]^m \!\!\!\!\!\nearrow$ for $\mathbf{i}=(i_1, \dots, i_m) \in O'$.
If the $i^\text{th}$ component of the output $q_i$ was to be sampled from a data interval that doesn't admit $q$ in its closure, then $\|\mathbf{q} - F_{X}^{-1}(\mathbf{p})\|_\infty \geq 
\eta$.
If on the other hand $q_i$ was to be sampled from a data interval that does admit $q$ in its closure, then it belongs to an interval $[X_{k}, X_{k +1})$ for some $k$ such that $q \in [X_{k}, X_{k +1}]$ and $X_{k +1} - X_{k } \geq 
\eta$. Conditionally to the fact that there are $m' \leq m$ other quantiles that are sampled from $[X_{k}, X_{k +1}]$, the conditional expectation of $\|\mathbf{q} - F_{X}^{-1}(\mathbf{p})\|_\infty$ can be lower by a (strictly) positive functional ($f(\eta, m')$) of $\eta$ and $m'$ (because the corresponding slice of the output is uniform on $[X_{k}, X_{k +1}]^{m'} \nearrow$. This shows that the risk can be lower bounded by a quantity in $\text{Conv}\{ \eta, f(\eta, 1), \dots, f(\eta, m)\}$ which is then bigger than $\text{min}\{ \eta, f(\eta, 1), \dots, f(\eta, m)\}$ which is positive.

\subsection{Proof of \Cref{preprocessionglemma}}
\label{preprocessinglemmaproof}

    Let $\mathcal{A}$ be a $\epsilon$-DP algorithm on $\mathfrak{X}^n$, $\mathbf{X}, \mathbf{X}' \in \mathfrak{X}^n$ such that 
    $\mathbf{X} \sim \mathbf{X}'$. Then, for every $\mathbf{w} \in \mathbb{R}^n, \text{proj}_{\mathfrak{X}^n}(\mathbf{X}+\mathbf{w}) \sim \text{proj}_{\mathfrak{X}^n}(\mathbf{X}'+\sigma(\mathbf{w}))$ for a specific permutation of the components $\sigma$.
    For each measurable set $\mathcal{S}\subseteq O$ we get
    \begin{equation*}
        \begin{aligned}
            &\mathbb{P}(\mathcal{A}(\text{proj}_{\mathfrak{X}^n}(\mathbf{X}+\mathbf{w})) \in \mathcal{S}) \\
            &=\int_{\mathbb{R}^n} \mathbb{P}_{\mathcal{A}}(\mathcal{A}(\text{proj}_{\mathfrak{X}^n}(\mathbf{X}+\mathbf{w})) \in \mathcal{S}) \mathbb{P}_{\mathbf{w}}(d \mathbf{w})\\
            &\leq e^\epsilon \int_{\mathbb{R}^n} \mathbb{P}_{\mathcal{A}}(\mathcal{A}(\text{proj}_{\mathfrak{X}^n}(\mathbf{X}'+\sigma(\mathbf{w}))) \in \mathcal{S}) \mathbb{P}_{\mathbf{w}}(d \sigma(\mathbf{w}))\\
            &= e^\epsilon \int_{\mathbb{R}^n} \mathbb{P}_{\mathcal{A}}(\mathcal{A}(\text{proj}_{\mathfrak{X}^n}(\mathbf{X}'+\mathbf{w})) \in \mathcal{S}) \mathbb{P}_{\mathbf{w}}(d \mathbf{w})\\
            &= e^\epsilon \mathbb{P}(\mathcal{A}(\text{proj}_{\mathfrak{X}^n}(\mathbf{X}'+\mathbf{w})) \in \mathcal{S})
        \end{aligned}
    \end{equation*}
    which completes the proof.

\subsection{Proof of \Cref{quantiledeviation}}
\label{quantiledeviationproof}

Let $t \geq 0$ such that $1-f(t) > 0$,
    \begin{equation*}
        \begin{aligned}
            &\mathbb{P}\left(X + w \leq F_{X}^{-1}\left( \frac{p}{1 - f(t)}\right) + t \right) \\
            &\geq \mathbb{P}\left(X + w \leq F_{X}^{-1}\left( \frac{p}{1 - f(t)}\right) + t, |w| \leq t \right)\\
            &\geq \mathbb{P}\left(X \leq F_{X}^{-1}\left( \frac{p}{1 - f(t)}\right), |w| \leq t \right)\\
            &\geq \mathbb{P}\left(X \leq F_{X}^{-1}\left( \frac{p}{1 - f(t)}\right)\right) \mathbb{P}\left(|w| \leq t \right)\\
            &\geq \frac{p}{1 - f(t)} \left(1 - f(t) \right) \geq p \;. \\
        \end{aligned}
    \end{equation*}
    So, $F_{\tilde{X}}^{-1}\left( p\right) \leq F_{X}^{-1}\left( \frac{p}{1 - f(t)}\right) + t$.
Let $\delta \in (0, p)$, the same arguments give
\begin{equation*}
        \begin{aligned}
            \mathbb{P}\left(X + w \leq -F_{-X}^{-1}\left( \frac{1-p + \delta}{1 - f(t)}\right) - t \right) \leq p-\delta < p
        \end{aligned}
    \end{equation*}
    which allows concluding with the desired result.

\subsection{Proof of \Cref{convThm}}
\label{convThmproof}

\begin{lemma}
\label[lemma]{chern1}
    Let $\tilde{X}$ be a real random variable with density $\pi_{\tilde{X}}$ and $p \in (0, 1)$. We suppose that $\pi_{\tilde{X}}\geq \pi_{\text{min}} > 0$ on an open neighborhood $\mathcal{N}$ of $F_{\tilde{X}}^{-1}(p)$. If we have access to $\mathbf{\tilde{X}} = (\tilde{X}_1, \dots, \tilde{X}_n)$ i.i.d. realisations of $\tilde{X}$ then for every $\gamma > 0$, if $n \geq \frac{2}{\gamma \pi_{\text{min}}}$,
    \begin{equation*}
        [F_{\tilde{X}}^{-1}(p)- \gamma, F_{\tilde{X}}^{-1}(p)+\gamma] \subset \mathcal{N}
        \implies \mathbb{P}\left( \left|F_{\tilde{X}}^{-1}(p) - \tilde{X}_{(\lceil np \rceil)} \right| > \gamma \right) \leq
        e^{-n \left(\frac{\gamma^2\pi_{\text{min}}^2}{8(1-p)}\right)} + e^{-n \left(\frac{\gamma^2\pi_{\text{min}}^2}{8p}\right)}
    \end{equation*}
\end{lemma}
\begin{proof}
Let $\gamma > 0$ such that $[F_{\tilde{X}}^{-1}(p)- \gamma, F_{\tilde{X}}^{-1}(p)+\gamma] \subset \mathcal{N}$. Let us define
\begin{equation*}
    N = \sum_{i=1}^n \one_{(F_{\tilde{X}}^{-1}(p) + \gamma, +\infty)}(\tilde{X}_i) \;.
\end{equation*}
$N$ is a sum of $n$ independent Bernoulli random variable with probabilities of success lower than $\eta = 1-p - \gamma \pi_{\text{min}}$.
If $\tilde{X}_{( \lceil np \rceil)} > F_{\tilde{X}}^{-1}(p) + \gamma$, then $N \geq n (1 - p)$. So,
\begin{equation*}
    \begin{aligned}
        & \mathbb{P}\left(\tilde{X}_{( \lceil np \rceil)} > F_{\tilde{X}}^{-1}(p) + \gamma \right)
        \leq \mathbb{P}\left(N \geq n (1 - p) - 1\right) \\
        &= \mathbb{P}\left(N \geq n \eta \left(1 + \frac{\gamma \pi_{\text{min}}}{\eta} - \frac{1}{n \eta}\right) \right)\\
        &\leq e^{-n \eta \left(\frac{\gamma \pi_{\text{min}}}{\eta} - \frac{1}{n \eta}\right)^2 / \left(2 + \frac{\gamma \pi_{\text{min}}}{\eta} - \frac{1}{n \eta}\right)}
    \end{aligned}
\end{equation*}
where line $3$ is deduced from line $2$ by a multiplicative Chernoff bounds. 
If we further impose that $n \geq \frac{2}{\gamma \pi_{\text{min}}}$,
\begin{equation*}
    \begin{aligned}
        & \mathbb{P}\left(\tilde{X}_{( \lceil np \rceil)} > F_{\tilde{X}}^{-1}(p) + \gamma \right)
        \leq e^{-\frac{n \eta}{4} \left(\frac{\gamma \pi_{\text{min}}}{\eta} \right)^2 / \left(2 + \frac{\gamma \pi_{\text{min}}}{\eta} \right)} \\
        &\leq e^{-\frac{n}{4} \left(\frac{\gamma^2\pi_{\text{min}}^2}{2(1-p)-\gamma \pi_{\text{min}}}\right)}
        \leq e^{-n \left(\frac{\gamma^2\pi_{\text{min}}^2}{8(1-p)}\right)}
    \end{aligned}
\end{equation*}

Looking at the event $\left(\tilde{X}_{(np)} < F_{\tilde{X}}^{-1}(p) - \gamma \right)$ and a union bound give the expected result.
\end{proof}

\begin{lemma}
\label[lemma]{chern2}
    Let $\tilde{X}$ be a real random variable with density $\pi_{\tilde{X}}$ and $p \in (0, 1)$. We suppose that $\pi_{\text{max}} \geq \pi_{\tilde{X}}\geq \pi_{\text{min}} > 0$ on an interval $I$ of $\mathbb{R}$. If we note $N = \sum_{i=1}^n \one_{I}(\tilde{X}_i)$ the number of points that fall in $I$, we have
    \begin{equation*}
        \mathbb{P}\left( N \geq 2 n \lambda(I) \pi_{\text{max}}\right) \leq e^{-\frac{n \lambda(I) \pi_{\text{max}}}{3}} \;,
    \end{equation*}
    \begin{equation*}
        \mathbb{P}\left( N \leq \frac{1}{2} n \lambda(I) \pi_{\text{min}}\right) \leq  e^{-\frac{n \lambda(I) \pi_{\text{min}}}{8}}\;.
    \end{equation*} \;.
\end{lemma}
\begin{proof}
This is a simple application of multiplicative Chernoff bounds to the sum N of independent Bernoulli random variables.
\end{proof}
Let $0< \gamma < \beta$ such that $\pi_{\tilde{X}} > 0$ on $\mathcal{O} := \cup_{i=1}^{n} [F_{\tilde{X}}^{-1}(p_i)-\beta, F_{\tilde{X}}^{-1}(p_i)+\beta]$.
We note $\pi_{\text{min}} = \inf_{\mathcal{O}} \pi_{\tilde{X}}$ and $\pi_{\text{max}} = \sup_{\mathcal{O}} \pi_{\tilde{X}}$. We also define the following events:
\begin{equation*}
    A : \forall i, \left|\tilde{X}_{(\lceil n p_i \rceil)} - F_{\tilde{X}}^{-1}(p_i)\right| \leq \gamma \;,
\end{equation*}
\begin{equation*}
\begin{aligned}
    B : \forall i, &\# (\tilde{\mathbf{X}} \cap [F_{\tilde{X}}^{-1}(p_i) + \gamma, F_{\tilde{X}}^{-1}(p_i) + \beta]) \geq \frac{1}{2} n (\beta-\gamma) \pi_{\text{min}} \text{ and } \\ &\# (\tilde{\mathbf{X}} \cap [F_{\tilde{X}}^{-1}(p_i) - \beta, F_{\tilde{X}}^{-1}(p_i) - \gamma ]) \geq \frac{1}{2} n (\beta-\gamma) \pi_{\text{min}}\;,
\end{aligned}
\end{equation*}
\begin{equation*}
    C : \forall i, \# (\tilde{\mathbf{X}} \cap [F_{\tilde{X}}^{-1}(p_i)-\gamma, F_{\tilde{X}}^{-1}(p_i) + \gamma]) \leq 2 n 2\gamma \pi_{\text{max}} \;.
\end{equation*}
Then we can compute,
\begin{equation*}
    \begin{aligned}
        \frac{\mathbb{P}\left( \|\mathcal{E}_{u_{\text{JE}}}^{(2/\epsilon)}(\tilde{\mathbf{X}}) - F_{\tilde{X}}^{-1}(\mathbf{p})\|_{\infty} > \beta | A, B, C \right)}
        {\mathbb{P}\left( \|\mathcal{E}_{u_{\text{JE}}}^{(2/\epsilon)}(\tilde{\mathbf{X}}) - F_{\tilde{X}}^{-1}(\mathbf{p})\|_{\infty} \leq \beta | A, B, C \right)}
        &\leq \frac{\mathbb{P}\left( \|\mathcal{E}_{u_{\text{JE}}}^{(2/\epsilon)}(\tilde{\mathbf{X}}) - F_{\tilde{X}}^{-1}(\mathbf{p})\|_{\infty} > \beta | A, B, C \right)}
        {\mathbb{P}\left( \|\mathcal{E}_{u_{\text{JE}}}^{(2/\epsilon)}(\tilde{\mathbf{X}}) - F_{\tilde{X}}^{-1}(\mathbf{p})\|_{\infty} \leq \gamma | A, B, C \right)} \\
    \end{aligned}
\end{equation*}
Conditionally to $A$ and $B$, $-u_{\text{JE}}(\tilde{\mathbf{X}}, \mathcal{E}_{u_{\text{JE}}}^{(2/\epsilon)}(\tilde{\mathbf{X}})) \leq \frac{1}{2}\left( \frac{1}{2} n (\beta-\gamma) \pi_{\text{min}}\right)  \implies \|\mathcal{E}_{u_{\text{JE}}}^{(2/\epsilon)}(\tilde{\mathbf{X}}) - F_{\tilde{X}}^{-1}(\mathbf{p})\|_{\infty} \leq \beta$.
Furthermore, conditionally to $A$ and $C$, $\|\mathcal{E}_{u_{\text{JE}}}^{(2/\epsilon)}(\tilde{\mathbf{X}}) - F_{\tilde{X}}^{-1}(\mathbf{p})\|_{\infty} \leq \gamma \implies -u_{\text{JE}}(\tilde{\mathbf{X}}, \mathcal{E}_{u_{\text{JE}}}^{(2/\epsilon)}(\tilde{\mathbf{X}})) \leq \frac{1}{2}\left( 4 (m+1) n \gamma \pi_{\text{max}}\right)$. So,
\begin{equation*}
\begin{aligned}
        \frac{\mathbb{P}\left( \|\mathcal{E}_{u_{\text{JE}}}^{(2/\epsilon)}(\tilde{\mathbf{X}}) - F_{\tilde{X}}^{-1}(\mathbf{p})\|_{\infty} > \beta | A, B, C \right)}
        {\mathbb{P}\left( \|\mathcal{E}_{u_{\text{JE}}}^{(2/\epsilon)}(\tilde{\mathbf{X}}) - F_{\tilde{X}}^{-1}(\mathbf{p})\|_{\infty} \leq \gamma | A, B, C \right)}
        &\leq 
        \frac{(b-a)^m}{(2\gamma)^m/m!} e^{-\frac{\epsilon}{4} \left(\frac{(\beta-\alpha) \pi_{\text{min}}}{2} -4 (m+1)\gamma \pi_{\text{max}} \right)n}
    \end{aligned}
\end{equation*}
and by fixing $\gamma = \frac{\beta \pi_{\text{min}}}{16 (m+1) \pi_{\text{max}} + 2 \pi_{\text{min}}}$ we end up with 
\begin{equation*}
\begin{aligned}
        &\frac{\mathbb{P}\left( \|\mathcal{E}_{u_{\text{JE}}}^{(2/\epsilon)}(\tilde{\mathbf{X}}) - F_{\tilde{X}}^{-1}(\mathbf{p})\|_{\infty} > \beta | A, B, C \right)}
        {\mathbb{P}\left( \|\mathcal{E}_{u_{\text{JE}}}^{(2/\epsilon)}(\tilde{\mathbf{X}}) - F_{\tilde{X}}^{-1}(\mathbf{p})\|_{\infty} \leq \beta | A, B, C \right)} \\
        &\leq 
        \frac{2^m (b-a)^m m!}{\beta^m} \left( \frac{4 (m+1) \pi_{\text{max}} + \pi_{\text{min}}/2}{\pi_{\text{min}}}\right)^m e^{- \frac{\epsilon \beta \pi_{\text{min}}}{16}n} \;.
    \end{aligned}
\end{equation*}

We can use \Cref{chern1}, \Cref{chern2} and union bounds to obtain the following for $n$ big enough:
\begin{equation*}
    \begin{aligned}
            &\mathbb{P}\left( \|\mathcal{E}_{u_{\text{JE}}}^{(2/\epsilon)}(\tilde{\mathbf{X}}) - F_{\tilde{X}}^{-1}(\mathbf{p})\|_{\infty} > \beta \right) \\
            &\leq \mathbb{P}\left( \|\mathcal{E}_{u_{\text{JE}}}^{(2/\epsilon)}(\tilde{\mathbf{X}}) - F_{\tilde{X}}^{-1}(\mathbf{p})\|_{\infty} > \beta |A, B, C\right) + \mathbb{P}(A^c) + \mathbb{P}(B^c) + \mathbb{P}(C^c) \\
            &\leq \frac{2^m (b-a)^m m!}{\beta^m} \left( \frac{4 (m+1) \pi_{\text{max}} + \pi_{\text{min}}/2}{\pi_{\text{min}}}\right)^m e^{- \frac{\epsilon \beta \pi_{\text{min}}}{16}n}\\
            &+ \sum_{i=1}^m e^{-n \left(\frac{\beta^2\pi_{\text{min}}^4}{8 (1-p_i) \left(16  (m+1)  \pi_{\text{max}} + 2 \pi_{\text{min}}\right)^2}\right)} + \sum_{i=1}^m e^{-n \left(\frac{\beta^2\pi_{\text{min}}^4}{8 p_i \left(16  (m+1)  \pi_{\text{max}} + 2 \pi_{\text{min}}\right)^2}\right)} \\
            &+ m e^{-n \frac{\beta \pi_{\text{min}} \pi_{\text{max}}}{24 (m+1) \pi_{\text{max}} + 3 \pi_{\text{min}}}} 
            + 2 m e^{-n \frac{\pi_{\text{min}}}{8}\left( \beta - \frac{\beta \pi_{\text{min}}}{16 (m+1) \pi_{\text{max}} + 2 \pi_{\text{min}}}\right)} \;.
    \end{aligned}
\end{equation*}

\subsection{Proof of \Cref{convThmHSJE}}
\label{convThmHSJEproof}

We tune the noise $w$ to have density $d\mathbb{P}_w(w) = \frac{\one_{[-\delta/2, \delta/2]}(w)}{\delta} dw$.
Under the hypothesis, $F_X^{-1}$ has a finite number of discontinuity points. We can apply \Cref{quantiledeviation} with $t=\delta/2$ and $f(t)=0$ to get that for Lebesgue-almost-any $\mathbf{p}$, 
\begin{equation*}
    \|F_{\tilde{X}}^{-1}(\mathbf{p}) - F_X^{-1}(\mathbf{p})\|_\infty
    \leq \delta/2 \;.
\end{equation*}
In order to conclude, we can describe the density $\pi_{\tilde{X}}$ of the noisy random variable. It is piecewise continuous on $[a, b]$, $\pi_{\tilde{X}}>0$ on $[a, b] \setminus \mathcal{O}'$ where $\mathcal{O}'$ is a finite union of intervals and $\pi_{\tilde{X}}=0$ on $\mathcal{O}'$. Consequently, there only are a finite number of $p$'s in $(0,1)$ such that it is not possible to find a $\beta > 0$ such that $\pi_{\tilde{X}}>0$ on  $[F_{\tilde{X}}^{-1}(p)-\beta, F_{\tilde{X}}^{-1}(p)+\beta]$ and where $\pi_{\tilde{X}}$ is continuous on that interval. Any $\mathbf{p}$ that has no such $p$ as any of its components qualifies and we can apply \Cref{convThm} to get that 
\begin{equation*}
    \|\mathbf{q} - F_{\tilde{X}}^{-1}(\mathbf{p})\|_\infty \leq \delta/2
\end{equation*}
with high probability.
We get the result by the triangle inequality.


\end{document}